\documentclass{article}

\usepackage[preprint]{classes/neurips_2024}

\usepackage[utf8]{inputenc} 
\usepackage[T1]{fontenc}    
\usepackage{hyperref}       
\usepackage{url}            
\usepackage{booktabs}       
\usepackage{amsfonts}       
\usepackage{nicefrac}       
\usepackage{microtype}      
\usepackage{xcolor}         

\usepackage{amsmath}
\usepackage{amssymb}
\usepackage{mathtools}
\usepackage{amsthm}
\usepackage{dsfont} 
\usepackage{bm} 

\usepackage{cleveref}

\RequirePackage{algorithm}
\RequirePackage{algorithmic}

\theoremstyle{plain}
\newtheorem{theorem}{Theorem}[section]

\newtheorem{lemma}[theorem]{Lemma}
\newtheorem{corollary}[theorem]{Corollary}
\theoremstyle{definition}
\newtheorem{definition}[theorem]{Definition}

\theoremstyle{remark}

\newcommand{\D}{\mathcal{D}}
\newcommand{\R}{\mathbb{R}}
\newcommand{\df}{\gamma}

\DeclarePairedDelimiterX\ip[2]{\langle}{\rangle}{#1,#2}

\let\P\undefined
\DeclarePairedDelimiterXPP\P[1]{\mathbb{P}}(){}{
    
    #1
}

\DeclarePairedDelimiterXPP\E[1]{\mathbb{E}}[]{}{
    
    #1
}

\DeclarePairedDelimiterXPP\Es[2]{\mathbb{E}_{#1}}[]{}{
    
    #2
}

\newcommand{\feat}[1][x,a]{\bm{\varphi}({#1})}

\newcommand{\rvec}{\bm{r}}
\newcommand{\pivec}{\pi}
\newcommand{\pfunc}[2][x']{p\left({#1}\middle|{#2}\right)}

\newcommand{\fvec}{\bm{\varphi}}

\renewcommand{\succeq}{\geq}

\DeclarePairedDelimiterXPP\Ipt[2]{{#1}\transpose}(){}{#2}
\DeclarePairedDelimiterXPP\Iptr[2]{}(){\transpose{#2}}{#1}

\newcommand{\Lag}{\mathfrak{L}}

\newcommand{\maximize}{\mathrm{maximize}}
\newcommand{\minimize}{\mathrm{minimize}}
\newcommand{\subjectto}{\mathrm{subject\ to}}

\usepackage{xspace}

\newcommand{\rtheta}{{\bm{\omega}}}

\newcommand{\FOGAS}{\textup{\textbf{\texttt{FOGAS}}}\xspace}

\renewcommand{\vec}[1]{{\boldsymbol{{#1}}}}

\newcommand{\psivec}{\bm{\psi}}
\newcommand{\hpsivec}{\widehat{\psivec}}
\newcommand{\phii}{\bm{\varphi}_{i}}
\newcommand{\Em}{\bm{E}}
\newcommand{\Pm}{\bm{P}}
\newcommand{\xvec}{\bm{x}}

\newcommand{\yvec}{\bm{y}}

\newcommand{\mvec}{\bm{m}}
\newcommand{\Mm}{\bm{M}}
\newcommand{\vvec}{{\bm{v}}}
\newcommand{\pvec}{{\bm{p}}}
\newcommand{\qvec}{{\bm{q}}}

\newcommand{\muvec}{\bm{\mu}}
\newcommand{\hmuvec}{\widehat{\muvec}}
\newcommand{\lvec}{\bm{\lambda}}
\newcommand{\tvec}{\bm{\theta}}

\newcommand{\Psim}{\bm{\Psi}}
\newcommand{\hPsim}{\widehat{\Psim}}
\newcommand{\Phim}{\bm{\Phi}}
\newcommand{\Lm}[1][n]{\bm{\Lambda}_{#1}}
\newcommand{\IIm}{\bm{I}_{n}}
\newcommand{\nuvec}{\bm{\nu}}

\newcommand{\gl}{\bm{g}_{\lvec}(t)}

\newcommand{\X}{\mathcal{X}}
\newcommand{\F}{\mathcal{F}}
 \newcommand{\A}{\mathcal{A}}
\newcommand{\N}{\mathcal{N}}

\newcommand{\bb}[2]{\mathbb{B}_{#1}(#2)}
\newcommand{\real}{\mathbb{R}}
\newcommand{\Rn}[0]{\mathbb{R}} 

\newcommand{\trace}[1]{\mbox{tr}\left(#1\right)}
\newcommand{\II}[1]{\mathbb{I}_{\left\{#1\right\}}}

\newcommand{\EEs}[2]{\mathbb{E}_{#2}\left[#1\right]}

\newcommand{\EEJ}[1]{\mathbb{E}_J\left[#1\right]}

\newcommand{\EEc}[2]{\mathbb{E}\left[#1\left|#2\right.\right]}

\newcommand{\EEcs}[3]{\mathbb{E}_{#3}\left[\left.#1\right|#2\right]}

\def\argmin{\mathop{\mbox{ arg\,min}}}
\def\argmax{\mathop{\mbox{ arg\,max}}}
\newcommand{\ra}{\rightarrow}

\newcommand{\iprod}[2]{\left\langle#1,#2\right\rangle}
\newcommand{\biprod}[2]{\bigl\langle#1,#2\bigr\rangle}

\newcommand{\norm}[1]{\left\|#1\right\|}
\newcommand{\abs}[1]{\left|#1\right|}
\newcommand{\onenorm}[1]{\norm{#1}_1}
\newcommand{\twonorm}[1]{\norm{#1}_2}
\newcommand{\sqtwonorm}[1]{\norm{#1}_2^{2}}

\newcommand{\infnorm}[1]{\norm{#1}_\infty}

\newcommand{\ev}[1]{\left\{#1\right\}}
\newcommand{\pa}[1]{\left(#1\right)}
\newcommand{\bpa}[1]{\bigl(#1\bigr)}

\newcommand{\sq}[1]{\pa{#1}^{2}}

\newcommand{\wh}{\widehat}

\newcommand{\hf}{\wh{f}}

\newcommand{\transpose}{^\mathsf{\scriptscriptstyle T}}

\usepackage{todonotes}
\definecolor{PalePurp}{rgb}{0.66,0.57,0.66}

\newcommand{\regret}{\mathfrak{R}}
\newcommand{\gap}{\mathfrak{G}}

\newcommand{\GER}{\text{err}_{\hPsim}}

\newcommand{\DDKL}[2]{\mathcal{D}_{\textup{KL}}\pa{#1\middle\|#2}}

\newcommand{\VV}{\mathcal{V}}

\newcommand{\Unif}[1]{\mathcal{U}({#1})}

\title{Offline RL via Feature-Occupancy Gradient Ascent}

\author{%
  Gergely Neu\\
  Universitat Pompeu Fabra\\
  Barcelona, Spain\\
  \texttt{gergely.neu@gmail.com}  \\
  \And
  Nneka Okolo \\
  Universitat Pompeu Fabra\\
  Barcelona, Spain\\
  \texttt{nnekamaureen.okolo@upf.edu} 
}

\begin{document}

\maketitle

\begin{abstract}
	We study offline Reinforcement Learning  in large 
	infinite-horizon discounted Markov Decision Processes (MDPs) when the 
	reward and transition models are linearly realizable under a known 
	feature map. 
	Starting from the classic linear-program formulation of the optimal control problem in MDPs, 
	we develop a new algorithm that performs a form of gradient ascent in the space of 
	feature occupancies, defined as the expected feature vectors that can potentially be generated 
	by executing policies in the environment.
	We show that the resulting simple algorithm satisfies
	strong computational and sample complexity guarantees, achieved under the least restrictive 
	data coverage assumptions known in the literature.
	In particular, we show that the sample complexity of our method scales optimally with the 
	desired accuracy level and depends on a weak notion of 
	coverage that 
	only requires the empirical feature covariance matrix to cover a single direction in the feature space (as 
    opposed to covering a full subspace). Additionally, our method is easy to implement and requires no prior knowledge 
    of the coverage ratio (or even an upper bound on it), which altogether make 
    it the strongest known algorithm for 
    this setting to date.
\end{abstract}

\section{Introduction}\label{sec:intro} 
We study Offline Reinforcement Learning (ORL) in sequential 
decision making problems whereby a learner aims to find a near-optimal 
policy with sole access to a static dataset of interactions with the underlying 
environment \citep{levine2020offline}. This line 
of work is naturally relevant to real-world tasks for which learning an 
accurate simulator of the environment is potentially intractable or 
impossible, trial-and-error learning could have grave consequences, yet 
logged interaction data is readily available. For example, in a high-stake 
application such as autonomous driving, building a sufficiently accurate 
simulator for the vehicle and its environment would require modelling very 
complex systems, which can be intractable both statistically and 
computationally. At the same time, running experiments in the real world could 
endanger the lives of other road users or result in damages to the vehicle. 
Yet, with the advent of tools for efficient sensory-data collection and 
processing, large volumes of logged data from human drivers are readily 
available.

An efficient ORL method is one which finds a near-optimal policy after a 
tractable number of elementary computations and samples from the dataset. 
It is well-known in this setting that the quality of the solution has to
heavily depend on the quality of the data, and in particular one cannot 
hope to find a near-optimal policy if the data covers the space of states 
and actions poorly. To formalize this intuition, many notions of data 
coverage have been proposed in the offline RL literature, ranging from 
a very restrictive uniform coverage assumption that requires the 
data-generating policy to cover the entire state-action space
\citep{munos2008finite} to a variety of partial coverage conditions
whereby this exploratory condition is only required for state-action pairs 
that are of interest to the optimal policy \citep{liu2020provably,  
rashidinejad2021bridging, uehara2021pessimistic, zhan2022offline,  
rashidinejad2022optimal,li2024settling}. In the present work, we study the 
setting of linear \emph{Markov Decision Processes} (MDPs) 
\citep{jin2020provably,yang2019sample} where the reward and transition matrix  
admit a low rank structure in terms of a known feature map, and data-coverage 
assumptions can be defined in the space of features. As shown by  
\citep{zanette2021provable}, in this setting it is possible to obtain strong 
guarantees if the offline data is well-aligned with the expectation of the 
feature vector generated by the optimal policy (as opposed to requiring 
alignment with the entire distribution of features as required by other common 
offline RL methods \citep{jin2021pessimism, xie2021bellman, 
uehara2021pessimistic, zhang2022corruption}). In the present paper, we propose 
a simple and efficient algorithm that yields the best known sample complexity 
guarantees for this problem setting, all while only requiring the weakest known 
data-coverage assumptions of \citet{zanette2021provable}.
 
Our approach is based on the LP formulation of optimal control in infinite-horizon 
discounted MDPs due to \cite{manne1960linear}, and more specifically on its low-dimensional 
saddle-point reparametrization for linear MDPs proposed by \citet{gabbianelli2023offline} 
(which itself builds on earlier work by \citealp{neu2023efficient} 
and \citealp{bas2021logistic}). Primal variables of this saddle-point objective correspond to 
expectations of feature vectors under the state-action distribution of each policy 
(called \emph{feature occupancies}), 
and dual variables correspond to parameters of linear approximations of action-value 
functions. We design an algorithm based on the idea of optimizing the unconstrained 
primal function that is derived from the saddle-point objective by eliminating the 
dual variables via a classic dualization trick. More precisely, we design 
a sample-based estimator of the primal function and optimize it via a variant 
of gradient ascent in the space of feature occupancies.

This approach is to be contrasted with the method of \citet{gabbianelli2023offline}, 
which instead optimized the original saddle-point objective via stochastic 
primal-dual methods. Their algorithm interleaved a sequence of ``policy improvement'' steps 
with an inner loop performing ``policy evaluation'', which resulted in a suboptimal 
use of sample transitions due to the costly inner loop. This issue was addressed in 
the very recent work of \citet{hong2024primal} who, instead of relying on stochastic 
optimization, built an estimator of the saddle-point objective and optimized it via 
a deterministic primal-dual method. Our approach is directly inspired by their idea
of estimating the saddle-point objective, but our algorithm design is significantly 
simpler: instead of directly optimizing the primal function in terms of feature occupancies,
\citet{hong2024primal} relied on a sophisticated reparametrization of the primal 
variables, and used a computationally involved procedure to update the dual variables.
Both of these steps required prior knowledge of a tight bound on the feature-coverage
ratio of the optimal policy, which is typically not available in problems of practical 
interest. Such knowledge is not required by our algorithm, thanks to the incorporation
of a recently proposed stabilization trick that we make use of in our algorithm 
\citep{jacobsen2023unconstrained,neu2024dealing}. We provide a more detailed discussion
of these closely related works in Section~\ref{sec:discussion}.
\paragraph{Notation.}
We use boldface lowercase letters $\mvec$ to denote vectors and and bold uppercase $\Mm$ for matrices. 
We define the Euclidean ball in $\Rn^{d}$ of radius $D$ by $\bb{d}{D} = \{\xvec\in\Rn^{d}|\twonorm{\xvec}\leq 
D\}$ and the $A$-simplex over a finite set $\A$ of cardinality $A$ as 
$\Delta_{\A} = \{p\in\Rn^{A}_{+}|\onenorm{p}=1\}$. 

\section{Preliminaries}\label{sec:prelim}
We consider infinite-horizon Discounted Markov Decision Processes (DMDPs) 
\citep{puterman2014markov} of the form $\pa{\X,\A,\rvec,\Pm,\df}$ where $\X$ 
denotes a finite (yet large) set of $X$ states and $\A$ is a finite action 
space of cardinality $A=|\A|$. 
We refer to $\rvec\in[0,1]^{XA}$ as the reward vector, $\Pm\in\Rn^{XA\times 
X}_{+}$ the transition matrix and $\df\in(0,1)$ the discount factor. For a 
state-action 
pair $\pa{x,a}\in\X\times\A$ we 
also use the notation $r\pa{x,a} =  \rvec[\pa{x,a}]$ to denote the reward of 
taking action $a$ in state $x$ and $p\pa{x'|x,a} = \Pm[\pa{x,a},x']$ as the 
probability of ending up in state $x'$ afterwards.

The MDP models a sequential decision making process where an agent interacts 
with its environment as follows. 
For each step $k=0,1,2,\cdots,$, the agent observes the current state $X_{k}$ 
of the environment and then goes on to select its 
action $A_{k}$. Based on this action in the current state, it receives a reward 
$r\pa{X_{k},A_{k}}$, transits to a new state $X_{k+1}\sim 
\pvec\pa{\cdot\middle|X_{k},A_{k}}$ and the process continues. 
The objective of the agent is to find a decision-making rule that maximizes 
its total discounted reward when the initial state $X_{0}$ is sampled according 
to a fixed initial-state distribution distribution $\nuvec_{0}\in\Delta_{\X}$. 
Without loss of generality, we assume that the initial state is fixed almost surely as $X_0 = x_0$, and use $\nuvec_0$ 
to refer to the corresponding delta distribution.
It is known that this objective can be achieved by executing a \emph{stationary stochastic policy}
$\pi:\X\ra\Delta_{\A}$, with $\pi(a|x)$ denoting the probability of the agent selecting
action $A_k = a$ in state $X_k = x$ for all $k$. We will use $\Pi$ to denote the set of all such behavior rules and 
will often simply call them \emph{policies}.
We define the normalized discounted return of 
each policy $\pi$ as 
\[
 \rho\pa{\pi} = \pa{1-\df}\EEs{\sum_{k=0}^{\infty}\gamma^{k}r\pa{x_{k},a_{k}}}{\nuvec_{0},\pi},
\]
where the role of the discount factor $\df\in(0,1)$ is to emphasize the importance of earlier rewards,
and the notation $\EEs{\cdot}{\nuvec_{0},\pi}$ highlights that the initial 
state is sampled from $\nu_0$ and
all actions are sampled according to the policy $\pi$. We will use $\pi^*$ to denote any policy that 
maximizes the return.

We will consider the offline RL setting where we are given access to a data set of $n$ sample transitions 
$\D_n = \{(X_{i}, A_{i}, R_{i}, X'_{i})\}_{i=1}^{n}$, where $X'_i \sim 
\pvec(\cdot|X_i,A_i)$ 
is sampled independently for each $i$ and  $R_i = r(X_i,A_i)$. Otherwise, no assumption is made about the 
state-action pairs $(X_i,A_i)$, and in particular we do not require these to be generated by a fixed behavior policy or 
to be independent of each other. \looseness=-1

For describing the approach we take towards solving this problem, we need to introduce some further standard notations.
The value function and action-value function associated with policy $\pi$ are respectively defined as 
\[
	v^{\pi}\pa{x} = 
	\EEs{q^{\pivec}\pa{x,a}}{a\sim\pivec\pa{\cdot|x}},
	\quad
	q^{\pi}\pa{x,a} = 
	\EEcs{\sum_{k=0}^{\infty}\gamma^{k}r\pa{x_{k},a_{k}}}{x_{0} = x, a_{0} = 
		a}{\pi},
\]
and the state-occupancy and state-action-occupancy measures under $\pivec$ as
\[
\nu^{\pi}\pa{x} = 
\sum_{a}\mu^{\pivec}\pa{x,a},
\quad
\mu^{\pi}\pa{x,a} = 
\pa{1-\df}\EEs{\sum_{k=0}^{\infty}\gamma^{k}\II{x_{k},a_{k}}}{\nuvec_{0},\pi}.
\]
The value functions and occupancy measures adhere to the following recursive equations, respectively termed the 
Bellman equation and Bellman flow condition \citep{bellman1966dynamic}: 
\[
\qvec^{\pi} = \rvec + \df\Pm\vvec^{\pi},
\qquad
\muvec^{\pi} = \pivec\circ[\pa{1-\df}\nuvec_{0} + \df\Pm\transpose\muvec^{\pi}].
\]
Here, the composition operation $\circ$ is defined so that for any policy $\pivec$ and 
state distribution $\nuvec\in\Rn^{X}$, we have $\pa{\pivec\circ\nuvec}\pa{x,a} 
= \pivec\pa{a|x}\nu\pa{x}$.
Notice that we can express the return of $\pi$ in terms of value functions and 
occupancy measures as $\rho\pa{\pivec} = \pa{1-\df}\iprod{\nuvec_{0}}{\vvec^\pi} = 
\iprod{\muvec^{\pivec}}{\rvec}$. On this note, for a given target accuracy 
$\varepsilon>0$, we say policy $\pivec$ is $\varepsilon$-optimal if it 
satisfies $	\iprod{\muvec^{\pivec^{*}} - \muvec^{\pivec}}{\rvec}\leq 
\varepsilon$.

In the present work, we well make use of the \emph{linear MDP} assumption due to 
\citet{jin2020provably,yang2019sample}, which is defined formally as follows:
\begin{definition}[Linear MDP]\label{def:linMDP} An MDP is called linear if 
both the transition
and reward functions can be expressed as a linear function of a given
feature map $\bm{\varphi}:\X\times\A\to\R^d$. That is, there exist 
$\psivec:\X\to\R^d$
and $\rtheta\in\R^d$ such that, for every $x,x'\in\X$ and $a\in\A$:
\begin{equation*}
	r(x,a) = \ip{\feat}{\rtheta}, \quad \pfunc{x,a} = 
	\ip{\feat}{\psivec(x')}.
\end{equation*}
We denote by $\Phim\in\R^{|\X\times\A|\times d}$ the feature matrix with 
rows given by $\feat\transpose$ and $\Psim\in\R^{d\times|\X|}$ as the 
weight matrix with columns $\psivec(x)$. Further, we will assume that 
$\twonorm{\rtheta}\leq \sqrt{d}$, that $\twonorm{\Psim\vvec}\leq B\sqrt{d}$ holds for all $\vvec\in[-B,B]$, and that 
all feature vectors satisfy $\twonorm{\feat}\leq R$ for some $R \ge 1$.
\end{definition}
An immediate consequence of this assumption is that the action-value function of any policy $\pivec$ can be written as 
a linear function of the features as $\qvec^{\pivec} = 
\Phim\tvec^{\pivec}$, with $\tvec^{\pivec}=\rtheta + 
\df\Psim\vvec^{\pivec}\in\Rn^{d}$. For the rest of the paper we explicitly 
assume that the feature matrix $\Phim$ is full rank -- which is enough to 
ensure 
uniqueness of $\tvec^{\pi}$. It is common to assume that the 
feature 
dimension $d\ll X$ such that the transition operator is low-rank. As common in this setting, we will suppose throughout 
the paper that the feature map $\Phim$ is known.

Our algorithm design will be based on the linear programming formulation of MDPs, first proposed in a number of papers 
in the 1960's \citep{manne1960linear,Ghe60,dEp63,Den70}. 
This formulation frames the problem of finding an optimal control policy as the following pair of primal and dual 
linear programs:

\noindent\begin{minipage}{.49\linewidth}
	\begin{align}\label{eq:LP}
		\begin{alignedat}{2}
			& \maximize  &\quad& \ip{\muvec}{\rvec} \\
			& \subjectto && \Em\transpose\muvec =(1-\df)\nuvec_0 + 
			\df\Pm\transpose\muvec \\
			&&& \muvec \succeq 0,
		\end{alignedat}
	\end{align}
\end{minipage}%
\ \ \ \ \ \ \ \ \ \  \vline
\begin{minipage}{.45\linewidth}
	\begin{align}\label{eq:LP-dual}
		\begin{alignedat}{2}
			& \minimize  &\quad& (1-\df)\ip{\nuvec_0}{\vvec} \\
			& \subjectto && \Em\vvec \succeq \rvec + \df\Pm\vvec.\\
			&\\
		\end{alignedat}
	\end{align}
\end{minipage}

Here, the operator $\Em\in\Rn^{XA\times X}$ is defined such that for each 
$x,a$ and vectors $\muvec\in\Rn^{XA}, \vvec\in\Rn^{X}$,
\[
\pa{\Em\transpose\muvec}\pa{x} = \sum_{a\in\A}\muvec\pa{x,a},
\qquad
\pa{\Em\vvec}\pa{x,a} = \vvec\pa{x}.
\]
It is known that the occupancy measure of an optimal policy 
$\muvec^{\pivec^{*}}$  is 
an optimal solution of the primal LP~\eqref{eq:LP}. In fact, the 
feasible set of the primal is precisely the space of valid state-action 
occupancy measures that can be induced by stationary policies.
Therefore, given any feasible solution $\muvec$, we can extract the inducing policy as $\pivec_\mu\pa{a|x} = 
\mu\pa{x,a}/\sum_{a'}\mu\pa{x,a'}$ when $\sum_{a}\mu\pa{x,a}\neq 0$. 
Likewise, the state value function of the 
optimal policy $\pi^*$ is an optimal 
solution to the dual LP.
That said, since the LP features $XA$ variables and constraints, it cannot be 
solved directly in large MDPs.

In view of the above limitations, we consider the following reduced version of the above 
intractable LPs due to \cite{gabbianelli2023offline} (see also 
\citealp{neu2023efficient,bas2021logistic}):

\noindent\begin{minipage}{.49\linewidth}
	\begin{align}\label{eq:LP-phi}
		\begin{alignedat}{2}
			& \maximize  &\quad& \ip{\lvec}{\rtheta} \\
			& \subjectto && \Em\transpose\muvec =(1-\df)\nuvec_0 + 
			\df\Psim\transpose\lvec \\
			&&& \lvec = \Phim\transpose\muvec \\
			&&& \muvec \succeq 0,
		\end{alignedat}
	\end{align}
\end{minipage}%
\ \ \ \ \ \ \ \ \ \  \vline
\begin{minipage}{.45\linewidth}
	\begin{align}\label{eq:LP-phi-dual}
		\begin{alignedat}{2}
			& \minimize  &\quad& (1-\df)\ip{\nuvec_0}{\vvec} \\
			& \subjectto && \Em\vvec \succeq \Phim\tvec \\
			&&& \tvec =\rtheta + \df\Psim\vvec.\\
			&\\
		\end{alignedat}
	\end{align}
\end{minipage}

In view of the second constraint of the primal LP~\eqref{eq:LP-phi}, $\lvec$ should be thought of as expectations of 
feature vectors under occupancy measures, and we thus refer to them as \emph{feature occupancy} vectors. Similarly, the 
second constraint of the dual LP~\eqref{eq:LP-phi-dual} suggests that $\tvec$ should be thought of as parameters of the 
approximate action-value function $\qvec_{\tvec} = \Phim\tvec = \Phim\pa{\rtheta + \df \Psim \vvec} = \rvec + \df \Pm 
\vvec$. We use $\lvec^{\pi^*} = \Phim\transpose\muvec^{\pi^*}$ to denote the 
feature occupancy associated with the optimal policy 
$\pi^*$ and $\tvec^{\pi^*}$ to denote the parameter-vector of the optimal 
action-value function $\qvec^{\pi^*}$.
The Lagrangian corresponding to the LPs is given as\looseness=-1
\begin{align}
	\Lag(\lvec,\muvec; \vvec,\tvec)
	\nonumber&= (1-\df)\ip{\nuvec_0}{\vvec} 
	+ \ip{\lvec}{\rtheta + \gamma \Psim\vvec - \tvec}
	+ \ip{\muvec}{\Phim\tvec - \Em\vvec} \\
	\label{eq:lag}
	&= \ip{\lvec}{\rtheta} + \ip{\vvec}{(1-\df)\nuvec_0 + 
		\df\Psim\transpose\lvec - \Em\transpose\muvec}
	+ \ip{\tvec}{\Phim\transpose\muvec - \lvec}.
\end{align}
It is easy to verify that by the linear MDP property, the feasible sets of 
the above LPs coincide with those of the original LPs in an appropriate sense, and their 
optimal solutions correspond to the optimal state-action occupancy measure and 
state-value function respectively (see Appendix~\ref{appx:prelim}).
 
In order to further reduce the complexity of the LPs above, we introduce a policy $\pi$ and 
parametrize the remaining high-dimensional variables $\vvec$ and $\muvec$ as
\begin{equation}\label{eq:best_response}
	v_{\tvec,\pivec}(s) = \sum_{a}\pi(a|s)\iprod{\tvec}{\feat},
	\qquad 
	\mu_{\lvec,\pivec}(x,a) = \pi(a|x)\Bigl[(1-\df)\nu_0(x) + 
	\df\ip{\psivec(x)}{\lvec}\Bigr].
\end{equation}
Plugging this choice back into the Lagrangian, we obtain the objective 
\begin{align}
	\nonumber f(\lvec,\pivec;\tvec)
	&= \Lag(\lvec,\muvec_{\lvec,\pivec}; \vvec_{\tvec,\pivec},\tvec)\\
	\label{eq:red_lag}&= (1-\df)\ip{\nuvec_0}{\vvec_{\tvec,\pivec}} 
	+ \ip{\lvec}{\rtheta + \gamma \Psim\vvec_{\tvec,\pivec} - \tvec}\\
	&= \ip{\lvec}{\rtheta} + 
	\ip{\tvec}{\Phim\transpose\muvec_{\lvec,\pivec} - 
		\lvec}.\nonumber
\end{align}
It is easy to see that for any $\pi$ and $\lvec^\pi = \Phi\transpose\muvec^\pi$, we have $f(\lvec^\pi,\pi;\tvec) = 
\iprod{\muvec^\pi}{\rvec}$ for all $\tvec\in\real^d$. Furthermore, whenever $\lvec\neq\lvec^\pi$ then the 
$\tvec$-player has a winning strategy that can force $\min_{\tvec} f(\lvec,\pi;\tvec) = -\infty$. This 
(informally) suggests that an optimal policy can be found by solving the unconstrained saddle-point optimization 
problem $\max_{\lvec\in\real^d,\pi\in\Pi} \min_{\tvec\in\real^d} 
f(\lvec,\pivec;\tvec)$. Furthermore, since the optimal policy can be written as 
$\pi^*(a|x) = \II{a=\text{argmax}_{b} 
\iprod{\tvec^{\pi^*}}{\varphi(x,b)}}$, it is sufficient to consider softmax 
policies of the form 
\[
\Pi\pa{D_{\pi}} = 
\Biggl\{\pi_{\theta}\pa{a|x} = 
\frac{e^{\iprod{\feat}{\tvec}}}{\sum_{a'}e^{\iprod{\bm{\varphi}\pa{x,a'}}{\tvec}}}
\,\Bigg|\,\tvec\in\bb{d}{D_{\pi}}\Biggr\},
\]
which can approximate $\pi^*$ to good precision when the diameter $D_{\pi}$ is set to be large enough. This 
parametrization effectively reduces the high-dimensional LP into a 
low-dimensional saddle-point optimization problem.
\section{Feature-occupancy gradient ascent for offline RL in linear 
MDPs}\label{sec:FOGAS}
A natural idea for developing RL methods is to build an empirical approximation of the function $f$ defined in the 
previous section, and use primal-dual methods to find a saddle-point of the resulting approximation. For offline RL, 
this approach has been explored by \citet{gabbianelli2023offline} and \citet{hong2024primal}. In this work, we develop 
an alternative approach that seeks to directly optimize the return by approximately maximizing the unconstrained primal 
function $f^*:\real^d \times \Pi$, defined for each feature-occupancy vector 
$\lvec$ and policy $\pi$ as
\[
 f^*(\lvec,\pi) = \min_{\tvec \in \bb{d}{D_{\tvec}}} f(\lvec,\pi;\tvec),
\]
for an appropriately chosen feasible set $\bb{d}{D_{\tvec}}$. Given the 
discussion in the previous section, 
maximizing this function with respect to $\lvec$ and $\pi$ is rightly expected to result in an optimal policy (which 
intuition will be made formal in our analysis). Notably, the so-called 
objective $f$ in Equation~\eqref{eq:red_lag} depends on 
the transition weight matrix $\Psim$ which is unknown in general. As we soon 
show, this matrix dominates the loss of the $\tvec$-player and $\lvec$-player.
Based on these observations, our approach consists of building a 
well-chosen estimator $\hf$ of $f$, and then 
maximizing the associated primal function $\hf^*$ defined as
\[
\hf^*(\lvec,\pi) = \min_{\tvec \in \bb{d}{D_{\tvec}}} \hf(\lvec;\tvec,\pi).
\]
The objective $\hf$ is built via a least-squares estimator inspired by the classic LSTD model estimate of 
\citet{BB96,PLTPL08}, which has been successfully used for analyzing finite-horizon linear MDPs in a variety of recent 
works (e.g., \citealp{jin2020provably,NPB20}). In particular, we fit an estimator $\hPsim$ of the true matrix $\Psim$ 
using samples from the dataset $\D_n = \{(X_{i}, A_{i}, R_{i}, 
X'_{i})\}_{i=1}^{n}$ as follows. Let 
$\phii = \bm{\varphi}(X_{i}, A_{i})$ denote the feature vector of 
$\pa{X_{i}, A_{i}}$ and $\Lm = \beta\IIm + 
\frac{1}{n}\sum_{i=1}^{n}\phii\phii\transpose$ the empirical feature covariance 
matrix. We define the regularized least squares estimate of $\Psim$ at $x\in\X$ 
as
\[
\hpsivec\pa{x} = \argmin_{\psivec\pa{x}\in\Rn^{d}} 
\frac{1}{n}\sum_{i=1}^{n}\sq{\iprod{\phii}{\psivec\pa{x}} - \II{x = 
		X_{i}'}} + \beta\sqtwonorm{\psivec\pa{x}},
\]
so that the estimate can be written as
\begin{equation}\label{eq:P_LSQE}
	\hPsim = 
	\sum_{x\in\X}\hpsivec(x)\vec{e}_{x}\transpose = 
	\frac{1}{n}\Lm^{-1}\sum_{i=1}^{n}\phii\vec{e}_{X_{i}'}\transpose.
\end{equation}
With this matrix at hand, we define $\hf$ as 
\[
\hf(\lvec,\pivec;\tvec)
= (1-\df)\ip{\nuvec_0}{\vvec_{\tvec,\pivec}} 
+ \ip{\lvec}{\rtheta + \gamma \hPsim\vvec_{\tvec,\pivec} - \tvec}\\
= \ip{\lvec}{\rtheta} + 
\ip{\tvec}{\Phim\transpose\hmuvec_{\lvec,\pivec} - 
	\lvec},
\]
where $\widehat{\mu}_{\lvec,\pivec}(x,a) =	\pi(a|x)\Bigl[(1-\df)\nu_0(x) + 
	\df\ip{\hpsivec(x)}{\lvec}\Bigr]$ is a sample-based approximation of $\mu_{\lvec,\pivec}$.

For the purpose of optimization, we will employ appropriately chosen versions of mirror ascent \citep{NY83,BT03} to 
iteratively optimize the primal variables. Denoting the iterates for each 
$t=1,2,\dots,T$ by $\lvec_t$ and 
$\pi_t$, and defining $\theta_t = \argmin_{\tvec \in \bb{d}{D_{\tvec}}} 
\hf(\lvec_t,\pi_t;\tvec)$, the updates are defined as 
follows. Using $\gl = \nabla_{\lvec_t} \hf^*(\lvec_t,\pi_t)$ to denote the 
gradient of $\hf^*$ 
with respect to the feature occupancies, the first set of variables is updated as\looseness=-1
\begin{equation}\label{eq:lambda_update}
	\lvec_{t+1}
	=
	\argmax_{\lvec\in\Rn^{d}}\Bigl\{\iprod{\lvec}{\gl} - 
\frac{1}{2\eta}\norm{\lvec - \lvec_{t}}^{2}_{\Lm^{-1}} - \frac{\varrho}{2}\norm{\lvec}^{2}_{\Lm^{-1}}\Bigr\},
\end{equation}
where the first regularization term acts as proximal regularization (necessary for mirror-ascent-style methods), and 
the second one has a stabilization effect whose role will be made clear later 
in the analysis. The resulting update can be written in closed form, and is equivalent to a preconditioned 
gradient-ascent step on $\hf^*$. The policies are updated in each state-action pair $x,a$ as
\[
 \pi_{t+1}(a|x) = \frac{\pi_{t}(a|x) e^{\alpha \iprod{\varphi(x,a)}{\tvec_t}}}{\sum_{a'}\pi_{t}(a'|x) e^{\alpha 
\iprod{\varphi(x,a')}{\tvec_t}}} = \frac{\pi_{1}(a|x) e^{\alpha 
\iprod{\varphi(x,a)}{\sum_{k=1}^t \tvec_k}}}{\sum_{a'}\pi_{1}(a'|x) e^{\alpha 
\iprod{\varphi(x,a')}{\sum_{k=1}^t \tvec_k}}},
\]
corresponding to performing an entropy-regularized mirror ascent step in each 
state $x$ (cf.~\citealp{NJG17}). We use the shorthand notation $\pi_{t+1} = 
\sigma\bpa{\alpha \Phi \sum_{k=1}^t \tvec_k}$ to 
denote the resulting softmax policy, and note that it is fully specified by a $d$-dimensional vector that can be 
stored compactly. After the final iterate is computed, the algorithm picks the index $J$ uniformly 
at random and outputs the policy $\pi_J$. We refer to the resulting algorithm as \emph{Feature-Occupancy Gradient 
AScent} (\FOGAS), and present its detailed pseudocode featuring the explicit expressions of $\lvec_t$ and $\tvec_t$ as 
Algorithm~\ref{alg:main1}.
\begin{algorithm}[t]
	\caption{Feature-Occupancy Gradient Ascent (\FOGAS)}\label{alg:main1}
	\begin{algorithmic}
		\STATE {\bfseries Input:} Learning rates $\alpha,\varrho,\eta$,
		initial points 
		$\lvec_{1}\in\Rn^{d}, \pi_1\in\Pi\pa{D_{\pivec}}, \bar{\tvec}_{0} = 
		\vec{0}$, and 
		dataset $\D_n$.
		\FOR{$t=1$ {\bfseries to} $T$}
		\STATE \emph{// Value-parameter update}
		\STATE Compute
		\STATE $\Phim\transpose\hmuvec_{\lvec_{t},\pivec_{t}} = 
		\pa{1-\df}\sum_{a}\pi_{t}(a|x_0)\fvec(x_0,a) + 
		\df\frac{1}{n}\sum_{i=1}^{n}\sum_{a}\pi_{t}(a|X_{i}')\bm{\varphi}(X_{i}',a)
		\iprod{\phii}{\Lm^{-1}\lvec_{t}}$
		\STATE $\tvec_{t} = \argmin_{\tvec\in\bb{d}{D_{\theta}}}
		\ip{\tvec}{\Phim\transpose\hmuvec_{\lvec_{t},\pivec_{t}} - 
			\lvec_{t}}$
		\STATE
		
		\STATE \emph{// Policy update}
		\STATE Update $\bar{\tvec}_{t} = \bar{\tvec}_{t-1} + \tvec_{t}$ 
		\STATE $\pivec_{t+1} = \sigma\pa{\alpha\Phim\bar{\tvec}_{t}}$
		\STATE
		
		\STATE \emph{// Feature-occupancy update}
		\STATE Compute $\hPsim\vvec_{\tvec_{t},\pivec_{t}} =		
		\frac{1}{n}\Lm^{-1}\sum_{i=1}^{n}\phii\vvec_{\tvec_{t},\pivec_{t}}(X_{i}')$
		\STATE Compute $\gl = \rtheta + 
		\df\hPsim\vvec_{\tvec_{t},\pivec_{t}} - \tvec_{t}$
		\STATE $\lvec_{t+1} = \frac{1}{1+\varrho\eta}\pa{\lvec_{t} + 
		\eta\Lm\gl}$
		\STATE
		
		\ENDFOR
		\STATE {\bfseries return} $\pivec_J$ with $J\sim \Unif{1,\cdots,T}$.
	\end{algorithmic}
\end{algorithm}

The following theorem states our main result regarding the performance of 
\FOGAS.
\begin{theorem}\label{thm:main1}
	Let $\pi_{1}$ be the uniform policy and $\lvec_{1}=\vec{0}$. Also set 
	$D_{\tvec} = \sqrt{d}/\pa{1-\df}$, $D_{\pi} = \alpha TD_{\tvec}$ 
	and $\delta > 0$. Suppose that we run \FOGAS for $T \ge \frac{2 R^2 n \log 
	A}{\log(1/\delta)}$ rounds with 
	parameters $\beta = R^{2}/dT$ as well as 
	\[
		\alpha 
		= \sqrt{\frac{2\sq{1-\df}\log A}{R^{2}dT}},\,\,\,\, 
		\varrho
		= 
		\df\sqrt{\frac{320d^{2}\log\pa{2T/\delta}}{\sq{1-\df}n}},\,\,\,\,
		\eta = 
		\sqrt{\frac{\sq{1-\df}}{27R^{2}d^{2}T}}.
	\]
	Then, with probability at least $1-\delta$, the following bound is satisfied for any comparator policy $\pi^*$ and 
the associated feature-occupancy vector $\lvec^{\pi^{*}} = 	
\Phi\transpose\muvec^{\pi^*}$:
	\begin{equation*}
		\EEJ{\iprod{\muvec^{\pi^{*}} - \muvec^{\pi_{J}}}{\rvec}}
		= 
		\mathcal{O}\pa{\frac{\norm{\lvec^{\pi^{*}}}^2_{\Lm^{-1}} + 1}{1-\df}
			\cdot \sqrt{\frac{d^{2}\log\pa{2T/\delta}}{n}}},
	\end{equation*}
	with the expectation taken with respect to the random index $J$.
\end{theorem}
The most important factor in the bound of Theorem~\ref{thm:main1} is $\norm{\lvec^*}^2_{\Lm^{-1}}$, which measures the 
extent to which the data $\D_n$ covers the comparator policy $\pi^*$ in feature space. We accordingly refer to this 
quantity as the \emph{feature coverage ratio} between the policy $\pi^*$ and the data set $\D_n$, and we discuss its 
relationship with other notions of data coverage in Section~\ref{sec:discussion}. Notably, the bound holds 
simultaneously for all comparator 
policies $\pi^*$, and thus it can be restated in an oracle-inequality form. On the same note, \FOGAS does not need any 
prior upper bounds on the comparator norm $\norm{\lvec^*}^2_{\Lm^{-1}}$, and in particular it does not project the 
iterates 
$\lvec_t$ to a bounded set. These nontrivial properties are enabled by a recently proposed 
stabilization trick due to \citet{jacobsen2023unconstrained} and \citet{neu2024dealing}, which amounts to augmenting 
the standard mirror-ascent update of Equation~\eqref{eq:lambda_update} with the regularization term 
$\frac{\varrho}{2}\norm{\lvec}^{2}_{\Lm^{-1}}$. Without this additional regularization, the bounds would feature an 
additional factor of the order $\frac{1}{T} \sum_{t=1}^T \norm{\lvec_t}^2_{\Lm^{-1}}$, which cannot be controlled 
without projecting the iterates and in any case make it impossible to prove a comparator-adaptive bound. We defer 
further discussion of the result to Section~\ref{sec:discussion}.

\section{Analysis}\label{sec:analysis}
This section is dedicated to proving our main result, Theorem~\ref{thm:main1}. While we have defined \FOGAS as a 
``primal-only'' algorithm above, its analysis will be most convenient if we regard it as a primal-dual algorithm with 
implicitly defined dual updates. In particular, we will view the updates of \FOGAS as a sequence of steps in a zero sum 
game between two teams of players: the \emph{max players} that control $\lvec_t$ and $\pi_t$, and the \emph{min player} 
that picks $\tvec_t$. The min player uses the simple \emph{best-response} 
strategy of picking $\tvec_t = 
\argmin_{\tvec\in \bb{d}{D_{\tvec}}}\hf(\lvec,\pi_t)$, and the other two 
players perform their updates via 
appropriate versions mirror ascent on their respective objectives. Importantly, the updates of the $\lvec$-player are 
based on the gradients of $\hf^*$, which satisfy
\[
 \gl = \nabla_{\lvec_t} \hf^*(\lvec_t,\pi_t) = \nabla_{\lvec_t} 
 \pa{\min_{\theta \in \bb{d}{D_{\tvec}}} 
\hf(\lvec_t,\pi_t;\tvec)} = \nabla_{\lvec_t} \hf(\lvec_t,\pi_t;\tvec_t),
\]
where the last equality follows from an application of Danskin's theorem. This property enables a major conceptual 
simplification that allows the interpretation of the updates as optimizing the unconstrained primal $\hf^*$ directly. 
We refer the interested reader to Chapter~6 of \citet{bertsekas1997nonlinear} for more context on such use of 
primal-dual analysis.

More concretely, we make use of an analysis technique first developed by \citet{neu2023efficient}, and further refined 
by \citet{gabbianelli2023offline} and \citet{hong2024primal}. The core idea is to introduce the \emph{dynamic duality 
gap} defined on a sequence of iterates $\{\pa{\lvec_{t},\pivec_{t},\tvec_{t}}\}_{t=1}^{T}$ produced by some iterative 
method, and a set of well-chosen \emph{comparators} $\pa{\lvec^{*},\pivec^{*};\{\tvec_{t}^{*}\}_{t=1}^{T}}$ as
\[
 \gap_T\pa{\lvec^{*},\pivec^{*};\{\tvec_{t}^{*}\}_{t=1}^{T}}	= \frac{1}{T}\sum_{t=1}^{T} 
\pa{f(\lvec^{*},\pivec^{*};\tvec_{t}) - 
	f(\lvec_{t},\pivec_{t};\tvec^{*}_{t})}.
\]
Similar to Lemma~4.1 of \cite{gabbianelli2023offline}, we show in 
Lemma~\ref{lem:gap_to_subopt} below that with an appropriate 
choice of the comparator points, we can relate the gap to the expected 
suboptimality of policy $\pi_{J}$ where $J\sim\Unif{1,\cdots,T}$. We leave the 
proof in Appendix~\ref{appx:gap_to_subopt}.
\begin{lemma}\label{lem:gap_to_subopt}
    Suppose that $D_{\tvec} = \sqrt{d}/(1-\df)$. Choose 
	$\pa{\lvec^{*},\pivec^{*},\tvec^{*}_{t}} 
	= 
	\pa{\Phim\transpose\muvec^{\pi^{*}}, \pivec^{*}, \tvec^{\pi_{t}}} \in 
	\Rn^{d}\times 
	\Pi\pa{D_{\pivec}}\times \bb{d}{D_{\theta}}$ for $t=1, \cdots, T$ where 
	$\muvec^{\pi^{*}}$ 
	is a valid occupancy measure induced by $\pi^{*}$. Then,
	\begin{align*}
		\EEJ{\iprod{\muvec^{\pi^{*}} - \muvec^{\pi_{J}}}{\rvec}}
		&= 
		\gap_T\pa{\Phim\transpose\muvec^{\pi^{*}},\pivec^{*},\{\tvec^{\pi_{t}}\}_{t=1}^{T}}.
	\end{align*}
\end{lemma}
We will show below that the dynamic duality gap can be written  in terms 
of the \emph{regrets} 
of each player and an additional term related to the estimation error of $\hf$,
and then proceed to provide bounds on all of these quantities. Specifically, 
the regrets of each player with respect to each of their 
respective comparators are defined as
\begin{align*}
		\regret_{T}\pa{\pivec^{*}}
		&=
		\sum_{t=1}^{T}
		{\sum_{x}\nu^{*}(x)\sum_{a}\pa{\pi^{*}(a|x) - 
				\pi_{t}(a|x)}q_{t}(x,a)},&\\
		\regret_{T}\pa{\lvec^{*}}
		&=
		\sum_{t=1}^{T}
		{\widehat{f}(\lvec^{*},\pivec_{t};\tvec_{t}) - 
		\widehat{f}(\lvec_{t},\pivec_{t};\tvec_{t})}
		=
		\sum_{t=1}^{T}
		{\ip{\lvec^{*} - \lvec_{t}}{\rtheta + \gamma
				\hPsim\vvec_{\tvec_{t},\pivec_{t}} - \tvec_{t}}},\\
		\regret_{T}\pa{\tvec^{*}_{1:T}}
		&=
		\sum_{t=1}^{T}
		{\widehat{f}(\lvec_{t},\pivec_{t};\tvec_{t}) - 
		\widehat{f}(\lvec_{t},\pivec_{t};\tvec^{*}_{t})}
		=
		\sum_{t=1}^{T}
		{\ip{\tvec_{t} - 
				\tvec^{*}_{t}}{\Phim\transpose\hmuvec_{\lvec_{t},\pivec_{t}} - 
				\lvec_{t}}}.
	\end{align*}
where $\nuvec^{*} = (1-\df)\nu_0(x) + \df\ip{\psivec(x)}{\lvec^{*}}$. Furthermore, we define the \emph{gap-estimation 
error} as
\begin{equation}\label{eq:GER}
	\GER = \sum_{t=1}^{T}
	\biprod{\lvec^{*}}{\bpa{\Psim - \hPsim}\vvec_{\tvec_{t},\pi_{t}}}
	+ \sum_{t=1}^{T}
	\biprod{\lvec_{t}}{\bpa{\hPsim - \Psim}\vvec_{\tvec_{t}^{*},\pivec_{t}}}.
\end{equation}
The following lemma rewrites the duality gap using the above terms.
\begin{lemma}\label{lem:gap_to_reg2}
	The dynamic duality gap satisfies
	\begin{align*}
		\gap_T(\lvec^{*},\pivec^{*},\tvec^{*}_{1:T})
		&= \frac{1}{T}\regret_{T}\pa{\pivec^{*}}
		+ \frac{1}{T}\regret_{T}\pa{\lvec^{*}}
		+ \frac{1}{T}\regret_{T}\pa{\tvec^{*}_{1:T}}
		+ \frac{\df}{T}\GER.
	\end{align*}
\end{lemma}
The proof directly follows from a straightforward calculation similar to the proof of 
Lemma~4.2 of \citet{gabbianelli2023offline} and Section E.1 of 
\cite{hong2024primal} 
which is reproduced in Appendix~\ref{appx:gap_to_reg2} for completeness. It remains to bound the regret of the 
players, as well as the gap-estimation error. An obstacle we need to face in 
the analysis is that our bound of the latter error term 
scale with $\frac{1}{T} \sum_{t=1}^T \norm{\lvec_t}^2_{\Lm^{-1}}$, which is undesirable
given our aspiration to achieve bounds that scale only with the comparator norm 
$\norm{\lvec^*}^2_{\Lm^{-1}}$ without requiring prior upper bounds on this quantity 
(that would enable us to project the iterates to a bounded domain). This challenge is 
addressed by making use of the stabilization
technique of \citet{jacobsen2023unconstrained} and \citet{neu2024dealing} in the updates for 
the $\lvec$-player, which effectively eliminates these problematic terms. We briefly outline the remaining parts of the 
analysis 
below.

\subsection{Regret analysis}
The regrets of each player are respectively controlled by the following three 
lemmas.
\begin{lemma}\label{lem:piregret}
 Suppose that $\nuvec^{*}\in\Delta_{\X}$. Let $\pi_{1}$ be the uniform policy 
 which selects all actions with equal probability in each state. Under 
 the conditions on the feature map in Definition~\ref{def:linMDP}, the regret 
 of the $\pi$-player against $\pivec^{*}$ 
 satisfies $\frac{1}{T}\regret_{T}\pa{\pivec^{*}}\leq
	\frac{\log A}{\alpha T} + \frac{\alpha 	R^{2}D_{\tvec}^{2}}{2}$.
\end{lemma}
The proof is a standard application of the analysis of exponential-weight 
updates, stated as Lemma~\ref{lem:MD}.
\begin{lemma}\label{lem:lambdaregret}
	Let $\lvec_{1}=\vec{0}$ and $C = 6\beta\pa{d + D_{\tvec}^{2}} + 
	3d\pa{1+RD_{\tvec}}^{2} + 3\df^{2} dR^{2}D_{\tvec}^{2}$. Then, the regret 
	of the $\lvec$-player against any comparator $\lvec^{*}\in\Rn^{d}$ satisfies
 \[
	\frac{1}{T}\regret_{T}\pa{\lvec^{*}} \le 
	\pa{\frac{1}{2\eta T} + \frac{\varrho 
		}{2}}\norm{\lvec^{*}}^{2}_{\Lm^{-1}}
	+ \frac{\eta C}{2}
	- \frac{\varrho}{2T}\sum_{t=1}^{T}\norm{\lvec_{t}}^{2}_{\Lm^{-1}}.
 \]
\end{lemma}
The proof (provided in Appendix~\ref{appx:lambdaregret}) follows from applying the standard analysis of 
composite-objective 
mirror descent due to \citet{duchi2010composite} (stated as 
Lemma~\ref{lem:COGDA} in the Appendix) and the bound 
$\norm{\Lm\gl}^{2}_{\Lm^{-1}} \le C$ on the weighted norm of the gradients for 
all $t$ provided in Lemma~\ref{lem:gnorm}. 
\begin{lemma}\label{lem:thetaregret}
Let $D_{\tvec}=\sqrt{d}/\pa{1-\df}$. The regret of the $\theta$-player 
satisfies 
$\frac{1}{T}\regret_{T}\pa{\tvec^{*}_{1:T}} \le 0$.
\end{lemma}
As we show in Appendix~\ref{appx:thetaregret}, the above statement 
holds trivially thanks to the ``best-response'' definition of $\theta_t$. This 
concludes our regret analysis.

\subsection{Bounding the gap-estimation error}
The following statement (proved in Appendix~\ref{appx:GER_proof}) provides a bound on $\GER$:
\begin{lemma}\label{lem:GER}
	Suppose that $\twonorm{\feat}\leq R$ for all $(x,a)\in\X\times\A$, 
	$D_{\tvec}=\sqrt{d}/(1-\df)$ and $\alpha=\sqrt{2\sq{1-\df}\log 
	A/R^{2}dT}$ to optimize $\regret_{T}\pa{\pivec^{*}}$. Then, for any 
	$T\ge \frac{2 R^2 	n \log A}{\log(1/\delta)}$  and 
	and $\xi\ge0$, the following holds with probability at least $1-\delta$:
	\[
	\GER
	\le
	\frac{1}{2\xi} \pa{{\norm{\lvec^*}_{\Lm^{-1}}^{2}} + 
	\frac{1}{T} \sum_{t=1}^{T}
		{\norm{\lvec_{t}}_{\Lm^{-1}}^{2}}} + T^{2}\xi \pa{\frac{320d^{2}\log\pa{2T/\delta}}{n\sq{1-\df}}}.
	\]
\end{lemma}

\subsection{The proof of Theorem~\ref{thm:main1}}
The proof follows from applying Lemmas~\ref{lem:gap_to_subopt} and 
Lemma~\ref{lem:gap_to_reg2} when $\pa{\lvec^{*},\pivec^{*},\tvec^{*}_{t}} 
= \pa{\Phim\transpose\muvec^{\pi^{*}}, \pivec^{*}, \tvec^{\pi_{t}}} \in 
\Rn^{d}\times \Pi\pa{D_{\pivec}}\times \bb{d}{D_{\theta}}$ for $t=1, \cdots, 
T$. Then, adding up the bounds stated in 
Lemmas~\ref{lem:piregret}--\ref{lem:GER} under the respective conditions, yields
\begin{align*}
	\EEJ{\iprod{\muvec^{\pi^{*}} - \muvec^{\pi_{J}}}{\rvec}}
	&\leq
	\sqrt{\frac{d\log\pa{1/\delta}}{n\sq{1-\df}}}
	+
	\pa{\frac{1}{2\eta T} + \frac{\varrho}{2} + \frac{\df}{2\xi T}}
	\norm{\lvec^{\pi^{*}}}^{2}_{\Lm^{-1}}
	+ \frac{\eta C}{2}\\
	&\quad
	+
	\pa{\frac{\df}{\xi T} - \varrho}\frac{1}{2T}\sum_{t=1}^{T}
	{\norm{\lvec_{t}}_{\Lm^{-1}}^{2}} + \df T\xi \pa{\frac{320d^{2}\log\pa{2T/\delta}}{n\sq{1-\df}}}.
\end{align*}
Then, setting $\rho = \frac{\gamma}{\xi T}$ simplifies the second term and 
eliminates the third term. The claim then follows after optimizing the 
hyperparameters, with the full details provided in 
Appendix~\ref{appx:main_proof}.
\section{Discussion}\label{sec:discussion} 
We discuss various aspects of our results below.

\paragraph{Relation with previous work.} As discussed in the introduction, our work draws heavily on previous 
contributions of \citet{gabbianelli2023offline} and \citet{hong2024primal}. In particular, our idea of building a 
least-squares estimator of the transition function is directly borrowed from the latter of these works, and our 
implicit update rule for $\theta_t$ is also inspired by their work to a good extent. Their approach, however, failed to 
reach the same degree of efficiency due to a number of suboptimal design choices. First, they used an alternative 
parametrization of the feature occupancies which only allowed them to work under a more restrictive 
coverage condition, so that their bounds depend on $\norm{\lvec^*}_{\Lm^{-2}}$ which can be much larger than the 
feature coverage ratio appearing in our bounds. Second, their algorithm required a prior upper bound on this coverage 
parameter, with the guarantees scaling with the bound rather than the actual coverage. Such bounds are typically 
difficult to obtain in practice. Third, the implementation of their algorithm required intricate computational steps 
necessitated by their feature-occupancy parametrization. Our work has successfully removed these limitations and 
reduced the complexity of their method, thanks to a new primal-only analysis style that we hope will find further uses 
in reinforcement learning.

\paragraph{Computational and statistical efficiency.} As can be inferred from our main result, the sample complexity of 
finding an $\varepsilon$-optimal policy using our algorithm is of the order 
$d^{2}\norm{\lvec^*}_{\Lm^{-1}}^2/\varepsilon^2\sq{1-\df}$, which is optimal in 
terms of scaling with $\varepsilon$. The rate can 
be improved to scale linearly with the feature coverage ratio $\norm{\lvec^*}_{\Lm^{-1}}$, if a tight upper bound is 
known on it which can be used for hyperparameter tuning. We find this scenario to be unlikely, and are curious to see 
if future work can attain this improved scaling without such prior knowledge. As for computational 
complexity, we point out that the cost of each iteration of our method scales linearly with the sample size $n$, due to 
having to compute the matrix-vector products $\hPsim\vvec_{\tvec_{t},\pivec_{t}}$. Indeed, the matrix $\hPsim$ is sparse 
with $n$ non-zero rows, and as such computing this product takes linear time in $n$. Since the iteration complexity of 
\FOGAS scales linearly with the sample size $n$, this makes for an overall runtime complexity of order $n^2$. This 
limitation is of course shared with all methods using the same least-squares transition estimator for the transition 
model, including all work that builds on \citet{jin2020provably}, but we 
nevertheless wonder if a substantial 
improvement is possible on this front.

\paragraph{Data coverage assumptions.} The only works we are aware of that scale with the feature-coverage ratio 
$\norm{\lvec^*}_{\Lm^{-1}}$ are due to \citet{zanette2021provable} and \citet{gabbianelli2023offline}. The latter 
work only achieves this bound under the assumption that the data is drawn i.i.d.~from a fixed behavior policy with 
known feature covariance matrix, which is a much more restricted setting that we consider here. Such assumptions 
are not needed by \citet{zanette2021provable}, however their results are restricted to the simpler finite-horizon MDP 
setting, and their algorithm is arguably more complex than ours. Using our notation, their approach can be interpreted 
as solving a ``pessimistic'' version of the the relaxed dual LP~\eqref{eq:LP-phi-dual} that features some additional 
quadratic constraints. This approach is not computationally viable for the infinite-horizon discounted case we 
consider, as it requires solving a fixed-point equation with respect to the estimated transition operator 
(cf.~\citealp{wei2021learning}). 

\paragraph{Possible extensions.} Our approach can be extended and generalized in a variety of ways. First, following 
\citet{gabbianelli2023offline}, we believe that it is straightforward to extend our analysis to undiscounted 
infinite-horizon MDPs. Second, we similarly believe that an extension to constrained MDPs is possible without major 
challenges, following \citet{hong2024primal}. We did not pursue these extensions because we believe that they add 
little additional insight. There are other potential directions that we did not explore because we 
found them to be too ambitious for the moment. These include extending our results beyond linear MDPs to other MDP 
models with linear function approximation, including MDPs with low inherent Bellman rank (which may be within reach of 
the current theory, c.f.~\citealp{zanette2020learning}), linearly $Q^\pi$-realizable MDPs (which are known to be 
challenging, c.f.~\citealp{weisz2022confident,weisz2024online}). Even more ambitiously, one can ask if it is possible 
to extend our methods to work under more general notions of function approximation. This looks very challenging given 
the central role of feature occupancies in our formalism, which are strictly tied to linear function approximation. We 
are nevertheless optimistic that the ideas presented in this work will find use in other contexts, possibly including 
nonlinear function approximation in the future.



 \begin{ack}
	This project has received funding from the European Research Council (ERC) 
	under the European Union’s Horizon 2020 research and innovation programme 
	(Grant agreement No.~950180). 
 \end{ack}

\bibliographystyle{abbrvnat}
\bibliography{references/references}

\begin{thebibliography}{45}
\providecommand{\natexlab}[1]{#1}
\providecommand{\url}[1]{\texttt{#1}}
\expandafter\ifx\csname urlstyle\endcsname\relax
  \providecommand{\doi}[1]{doi: #1}\else
  \providecommand{\doi}{doi: \begingroup \urlstyle{rm}\Url}\fi

\bibitem[Abbasi-Yadkori et~al.(2011)Abbasi-Yadkori, P{\'a}l, and
  Szepesv{\'a}ri]{abbasi2011improved}
Y.~Abbasi-Yadkori, D.~P{\'a}l, and C.~Szepesv{\'a}ri.
\newblock Improved algorithms for linear stochastic bandits.
\newblock \emph{Advances in neural information processing systems}, 24, 2011.

\bibitem[Bas-Serrano et~al.(2021)Bas-Serrano, Curi, Krause, and
  Neu]{bas2021logistic}
J.~Bas-Serrano, S.~Curi, A.~Krause, and G.~Neu.
\newblock Logistic q-learning.
\newblock In \emph{International Conference on Artificial Intelligence and
  Statistics}, pages 3610--3618. PMLR, 2021.

\bibitem[Beck and Teboulle(2003)]{BT03}
A.~Beck and M.~Teboulle.
\newblock Mirror descent and nonlinear projected subgradient methods for convex
  optimization.
\newblock \emph{Operations Research Letters}, 31\penalty0 (3):\penalty0
  167--175, 2003.

\bibitem[Bellman(1966)]{bellman1966dynamic}
R.~Bellman.
\newblock Dynamic programming.
\newblock \emph{science}, 153\penalty0 (3731):\penalty0 34--37, 1966.

\bibitem[Bertsekas(1997)]{bertsekas1997nonlinear}
D.~P. Bertsekas.
\newblock Nonlinear programming.
\newblock \emph{Journal of the Operational Research Society}, 48\penalty0
  (3):\penalty0 334--334, 1997.

\bibitem[Bradtke and Barto(1996)]{BB96}
S.~J. Bradtke and A.~G. Barto.
\newblock Linear least-squares algorithms for temporal difference learning.
\newblock \emph{Machine Learning}, 22:\penalty0 33--57, 1996.

\bibitem[Cesa-Bianchi and Lugosi(2006)]{cesa2006prediction}
N.~Cesa-Bianchi and G.~Lugosi.
\newblock \emph{Prediction, learning, and games}.
\newblock Cambridge university press, 2006.

\bibitem[de~Ghellinck(1960)]{Ghe60}
G.~de~Ghellinck.
\newblock Les probl\`emes de d\'ecisions s\'equentielles.
\newblock \emph{Cahiers du Centre d’\'Etudes de Recherche Op\'erationnelle},
  2:\penalty0 161--179, 1960.

\bibitem[Denardo(1970)]{Den70}
E.~V. Denardo.
\newblock On linear programming in a {Markov} decision problem.
\newblock \emph{Management Science}, 16\penalty0 (5):\penalty0 281--288, 1970.

\bibitem[d'Epenoux(1963)]{dEp63}
F.~d'Epenoux.
\newblock A probabilistic production and inventory problem.
\newblock \emph{Management Science}, 10\penalty0 (1):\penalty0 98--108, 1963.

\bibitem[Duchi et~al.(2010)Duchi, Shalev-Shwartz, Singer, and
  Tewari]{duchi2010composite}
J.~C. Duchi, S.~Shalev-Shwartz, Y.~Singer, and A.~Tewari.
\newblock Composite objective mirror descent.
\newblock In \emph{COLT}, volume~10, pages 14--26. Citeseer, 2010.

\bibitem[Freund and Schapire(1997)]{FS97}
Y.~Freund and R.~E. Schapire.
\newblock A decision-theoretic generalization of on-line learning and an
  application to boosting.
\newblock \emph{Journal of Computer and System Sciences}, 55:\penalty0
  119--139, 1997.

\bibitem[Gabbianelli et~al.(2024)Gabbianelli, Neu, Papini, and
  Okolo]{gabbianelli2023offline}
G.~Gabbianelli, G.~Neu, M.~Papini, and N.~M. Okolo.
\newblock Offline primal-dual reinforcement learning for linear mdps.
\newblock In \emph{International Conference on Artificial Intelligence and
  Statistics}, pages 3169--3177. PMLR, 2024.

\bibitem[Hong and Tewari(2024)]{hong2024primal}
K.~Hong and A.~Tewari.
\newblock A primal-dual algorithm for offline constrained reinforcement
  learning with low-rank mdps.
\newblock \emph{arXiv preprint arXiv:2402.04493}, 2024.

\bibitem[Jacobsen and Cutkosky(2023)]{jacobsen2023unconstrained}
A.~Jacobsen and A.~Cutkosky.
\newblock Unconstrained online learning with unbounded losses.
\newblock In \emph{International Conference on Machine Learning}, pages
  14590--14630. PMLR, 2023.

\bibitem[Jin et~al.(2020)Jin, Yang, Wang, and Jordan]{jin2020provably}
C.~Jin, Z.~Yang, Z.~Wang, and M.~I. Jordan.
\newblock Provably efficient reinforcement learning with linear function
  approximation.
\newblock In \emph{Conference on learning theory}, pages 2137--2143. PMLR,
  2020.

\bibitem[Jin et~al.(2021)Jin, Yang, and Wang]{jin2021pessimism}
Y.~Jin, Z.~Yang, and Z.~Wang.
\newblock Is pessimism provably efficient for offline {RL}?
\newblock In \emph{International Conference on Machine Learning}, pages
  5084--5096. PMLR, 2021.

\bibitem[Lattimore and Szepesv{\'a}ri(2020)]{lattimore2020bandit}
T.~Lattimore and C.~Szepesv{\'a}ri.
\newblock \emph{Bandit algorithms}.
\newblock Cambridge University Press, 2020.

\bibitem[Levine et~al.(2020)Levine, Kumar, Tucker, and Fu]{levine2020offline}
S.~Levine, A.~Kumar, G.~Tucker, and J.~Fu.
\newblock Offline reinforcement learning: Tutorial, review, and perspectives on
  open problems.
\newblock \emph{arXiv preprint arXiv:2005.01643}, 2020.

\bibitem[Li et~al.(2024)Li, Shi, Chen, Chi, and Wei]{li2024settling}
G.~Li, L.~Shi, Y.~Chen, Y.~Chi, and Y.~Wei.
\newblock Settling the sample complexity of model-based offline reinforcement
  learning.
\newblock \emph{The Annals of Statistics}, 52\penalty0 (1):\penalty0 233--260,
  2024.

\bibitem[Littlestone and Warmuth(1994)]{LW94}
N.~Littlestone and M.~Warmuth.
\newblock The weighted majority algorithm.
\newblock \emph{Information and Computation}, 108:\penalty0 212--261, 1994.

\bibitem[Liu et~al.(2020)Liu, Swaminathan, Agarwal, and
  Brunskill]{liu2020provably}
Y.~Liu, A.~Swaminathan, A.~Agarwal, and E.~Brunskill.
\newblock Provably good batch off-policy reinforcement learning without great
  exploration.
\newblock \emph{Advances in neural information processing systems},
  33:\penalty0 1264--1274, 2020.

\bibitem[Manne(1960)]{manne1960linear}
A.~S. Manne.
\newblock Linear programming and sequential decisions.
\newblock \emph{Management Science}, 6\penalty0 (3):\penalty0 259--267, 1960.

\bibitem[Munos and Szepesv{\'a}ri(2008)]{munos2008finite}
R.~Munos and C.~Szepesv{\'a}ri.
\newblock Finite-time bounds for fitted value iteration.
\newblock \emph{Journal of Machine Learning Research}, 9\penalty0 (5), 2008.

\bibitem[Nemirovski and Yudin(1983)]{NY83}
A.~Nemirovski and D.~Yudin.
\newblock \emph{Problem Complexity and Method Efficiency in Optimization}.
\newblock Wiley Interscience, 1983.

\bibitem[Neu and Okolo(2023)]{neu2023efficient}
G.~Neu and N.~Okolo.
\newblock Efficient global planning in large {MDPs} via stochastic primal-dual
  optimization.
\newblock In \emph{International Conference on Algorithmic Learning Theory},
  pages 1101--1123. PMLR, 2023.

\bibitem[Neu and Okolo(2024)]{neu2024dealing}
G.~Neu and N.~Okolo.
\newblock Dealing with unbounded gradients in stochastic saddle-point
  optimization.
\newblock \emph{arXiv preprint arXiv:2402.13903}, 2024.

\bibitem[Neu and Pike-Burke(2020)]{NPB20}
G.~Neu and C.~Pike-Burke.
\newblock A unifying view of optimism in episodic reinforcement learning.
\newblock In \emph{Advances in Neural Information Processing Systems}, pages
  1392--1403, 2020.

\bibitem[Neu et~al.(2017)Neu, Jonsson, and G{\'o}mez]{NJG17}
G.~Neu, A.~Jonsson, and V.~G{\'o}mez.
\newblock A unified view of entropy-regularized {Markov} decision processes.
\newblock \emph{arXiv preprint arXiv:1705.07798}, 2017.

\bibitem[Parr et~al.(2008)Parr, Li, Taylor, Painter-Wakefield, and
  Littman]{PLTPL08}
R.~Parr, L.~Li, G.~Taylor, C.~Painter-Wakefield, and M.~L. Littman.
\newblock An analysis of linear models, linear value-function approximation,
  and feature selection for reinforcement learning.
\newblock In \emph{Proceedings of the 25th international conference on Machine
  learning}, pages 752--759, 2008.

\bibitem[Puterman(1994)]{puterman2014markov}
M.~L. Puterman.
\newblock \emph{Markov decision processes: discrete stochastic dynamic
  programming}.
\newblock John Wiley \& Sons, 1994.

\bibitem[Rashidinejad et~al.(2021)Rashidinejad, Zhu, Ma, Jiao, and
  Russell]{rashidinejad2021bridging}
P.~Rashidinejad, B.~Zhu, C.~Ma, J.~Jiao, and S.~Russell.
\newblock Bridging offline reinforcement learning and imitation learning: A
  tale of pessimism.
\newblock \emph{Advances in Neural Information Processing Systems},
  34:\penalty0 11702--11716, 2021.

\bibitem[Rashidinejad et~al.(2022)Rashidinejad, Zhu, Yang, Russell, and
  Jiao]{rashidinejad2022optimal}
P.~Rashidinejad, H.~Zhu, K.~Yang, S.~Russell, and J.~Jiao.
\newblock Optimal conservative offline rl with general function approximation
  via augmented lagrangian.
\newblock \emph{arXiv preprint arXiv:2211.00716}, 2022.

\bibitem[Shalev-Shwartz and Ben-David(2014)]{shalev2014understanding}
S.~Shalev-Shwartz and S.~Ben-David.
\newblock \emph{Understanding machine learning: From theory to algorithms}.
\newblock Cambridge university press, 2014.

\bibitem[Uehara and Sun(2021)]{uehara2021pessimistic}
M.~Uehara and W.~Sun.
\newblock Pessimistic model-based offline reinforcement learning under partial
  coverage.
\newblock \emph{arXiv preprint arXiv:2107.06226}, 2021.

\bibitem[Vovk(1990)]{Vov90}
V.~Vovk.
\newblock Aggregating strategies.
\newblock In \emph{Proceedings of the third annual workshop on Computational
  learning theory (COLT)}, pages 371--386, 1990.

\bibitem[Wei et~al.(2021)Wei, Jahromi, Luo, and Jain]{wei2021learning}
C.-Y. Wei, M.~J. Jahromi, H.~Luo, and R.~Jain.
\newblock Learning infinite-horizon average-reward mdps with linear function
  approximation.
\newblock In \emph{International Conference on Artificial Intelligence and
  Statistics}, pages 3007--3015. PMLR, 2021.

\bibitem[Weisz et~al.(2022)Weisz, Gy{\"o}rgy, Kozuno, and
  Szepesv{\'a}ri]{weisz2022confident}
G.~Weisz, A.~Gy{\"o}rgy, T.~Kozuno, and C.~Szepesv{\'a}ri.
\newblock Confident approximate policy iteration for efficient local planning
  in $q^{\pi}$-realizable mdps.
\newblock \emph{Advances in Neural Information Processing Systems},
  35:\penalty0 25547--25559, 2022.

\bibitem[Weisz et~al.(2024)Weisz, Gy{\"o}rgy, and
  Szepesv{\'a}ri]{weisz2024online}
G.~Weisz, A.~Gy{\"o}rgy, and C.~Szepesv{\'a}ri.
\newblock Online rl in linearly $q^{\pi}$-realizable mdps is as easy as in
  linear mdps if you learn what to ignore.
\newblock \emph{Advances in Neural Information Processing Systems}, 36, 2024.

\bibitem[Xie et~al.(2021)Xie, Cheng, Jiang, Mineiro, and
  Agarwal]{xie2021bellman}
T.~Xie, C.-A. Cheng, N.~Jiang, P.~Mineiro, and A.~Agarwal.
\newblock Bellman-consistent pessimism for offline reinforcement learning.
\newblock \emph{Advances in neural information processing systems},
  34:\penalty0 6683--6694, 2021.

\bibitem[Yang and Wang(2019)]{yang2019sample}
L.~Yang and M.~Wang.
\newblock Sample-optimal parametric {Q}-learning using linearly additive
  features.
\newblock In \emph{International conference on machine learning}, pages
  6995--7004. PMLR, 2019.

\bibitem[Zanette et~al.(2020)Zanette, Lazaric, Kochenderfer, and
  Brunskill]{zanette2020learning}
A.~Zanette, A.~Lazaric, M.~Kochenderfer, and E.~Brunskill.
\newblock Learning near optimal policies with low inherent bellman error.
\newblock In \emph{International Conference on Machine Learning}, pages
  10978--10989. PMLR, 2020.

\bibitem[Zanette et~al.(2021)Zanette, Wainwright, and
  Brunskill]{zanette2021provable}
A.~Zanette, M.~J. Wainwright, and E.~Brunskill.
\newblock Provable benefits of actor-critic methods for offline reinforcement
  learning.
\newblock \emph{Advances in neural information processing systems},
  34:\penalty0 13626--13640, 2021.

\bibitem[Zhan et~al.(2022)Zhan, Huang, Huang, Jiang, and Lee]{zhan2022offline}
W.~Zhan, B.~Huang, A.~Huang, N.~Jiang, and J.~Lee.
\newblock Offline reinforcement learning with realizability and single-policy
  concentrability.
\newblock In \emph{Conference on Learning Theory}, pages 2730--2775. PMLR,
  2022.

\bibitem[Zhang et~al.(2022)Zhang, Chen, Zhu, and Sun]{zhang2022corruption}
X.~Zhang, Y.~Chen, X.~Zhu, and W.~Sun.
\newblock Corruption-robust offline reinforcement learning.
\newblock In \emph{International Conference on Artificial Intelligence and
  Statistics}, pages 5757--5773. PMLR, 2022.

\end{thebibliography}


\newpage
\appendix
\section*{\Large{Appendix}}
\section{Missing proofs of Section~\ref{sec:prelim}}\label{appx:prelim}
\subsection{Properties of the relaxed LP}\label{appx:feasibility}
In this section we prove a basic result about the feasible sets of the relaxed 
linear programs defined in 
Equations~\eqref{eq:LP-phi} and~\eqref{eq:LP-phi-dual}. We remark that similar 
results have been previously
shown in Proposition~4 of \cite{bas2021logistic} and Appendix~A.1 of 
\cite{neu2023efficient}. 
\begin{lemma}\label{lem:linMDP}
	Suppose that the MDP satisfies the linear MDP assumption in the sense of 
	Definition~\ref{def:linMDP}, 
	consider the relaxed linear programs~\ref{eq:LP-phi} and 
	\ref{eq:LP-phi-dual} and their respective feasible sets:
	\begin{align*}
		\mathcal{M}^{P}_{\Phim}
		&=
		\ev{\pa{\lvec,\muvec}\in\R^{d}\times\Rn_{+}^{XA}
			~\middle|~\Em\transpose\muvec =(1-\df)\nuvec_0 + 
			\df\Psim\transpose\lvec,\quad\lvec = \Phim\transpose\muvec},\\
		\mathcal{M}^{D}_{\Phim}
		&=
		\ev{\pa{\vvec,\tvec}\in\Rn^{X}\times\R^{d}
			~\middle|~\Em\vvec \succeq \Phim\tvec,\quad\tvec =\rtheta + 
			\df\Psim\vvec}.
	\end{align*}
	Then, the following statements hold:
	\begin{itemize}
		\item The set $\mathcal{M} = \ev{\muvec : \pa{\lvec,\muvec} \in 
			\mathcal{M}^{P}_{\Phim}}$ coincides with the 
		feasible set of the primal LP~\eqref{eq:LP}. Furthermore, for all 
		$\pa{\lvec^{*},\muvec^{*}}
		\in 
		\argmax_{\pa{\lvec,\muvec}\in\mathcal{M}^{P}_{\Phim}}\iprod{\lvec}{\rtheta}$,
		 we have that $\muvec^*$ is the 
		occupancy measure of an optimal policy.
		\item The set $\mathcal{V} = \ev{\vvec : \pa{\vvec,\tvec} \in 
		\mathcal{M}^{D}_{\Phim}}$ coincides with the 
		feasible set of the dual LP~\eqref{eq:LP-dual}. Furthermore, the 
		optimal value 
		function $\vvec^{\pi^*}$ and the 
		parameter vector $\tvec^{\pi^*}$ satisfying $\qvec^{\pi^*} = 
		\Phim\tvec^{\pi^*}$ satisfy 
		$(\vvec^{\pi^*},\tvec^{\pi^*}) \in \argmin_{\pa{\vvec,\tvec} 
			\in \mathcal{M}^{D}_{\Phim}} (1-\gamma)\iprod{\nuvec_0}{\vvec}$.
	\end{itemize}
\end{lemma}

\begin{proof}
	We first show that for any feasible point $\muvec$ of the LP~\eqref{eq:LP}, 
	the tuple $(\lvec,\muvec)$ is feasible for 
	the relaxed LP with $\lvec = \Phim\transpose\muvec$. This choice of $\lvec$ 
	satisfies the second primal constraint by 
	definition, so it remains to verify that the first constraint is also 
	satisfied. Indeed, this follows from
	\begin{align*}
		\Em\transpose\muvec - (1-\df)\nuvec_0 - 
		\df\Psim\transpose\lvec
		&= \Em\transpose\muvec - (1-\df)\nuvec_0 - 
		\df\Psim\transpose\Phim\transpose\muvec\\
		&= \Em\transpose\muvec - (1-\df)\nuvec_0 - 
		\df\Pm\transpose\muvec = 0,
	\end{align*}
	where we have used the linear MDP property to write 
	$\Psim\transpose\Phim\transpose = \Pm\transpose$ in the first 
	step and that $\muvec$ is a valid occupancy measure in the last one. 
	Conversely, supposing that 
	$\pa{\lvec,\muvec}\in\mathcal{M}^{P}_{\Phim}$ are feasible for the relaxed 
	LP, we have that
	\begin{align*}
		\Em\transpose\muvec - (1-\df)\nuvec_0 -
		\df\Pm\transpose\muvec
		&=
		\Em\transpose\muvec - (1-\df)\nuvec_0 -
		\df\Psim\transpose\Phim\transpose\muvec\\
		&=
		\Em\transpose\muvec - (1-\df)\nuvec_0 -
		\df\Psim\transpose\lvec = 0,
	\end{align*}
	thus verifying that $\muvec$ is indeed a valid occupancy measure.
	Optimality of $\pa{\lvec^{*},\muvec^{*}}$ follows from the fact that for 
	any 
	$\pa{\lvec,\muvec}\in\mathcal{M}^{P}_{\Phim}$, we can write the LP 
	objective as 
	$\iprod{\lvec}{\rtheta} = \iprod{\muvec}{\rvec}$ by the linear MDP 
	assumption, and the standard fact that 
	any solution $\muvec^*$ to the primal LP~\ref{eq:LP} is the occupancy 
	measure of an optimal policy (cf.~Theorem 6.9.4 
	in \citealp{puterman2014markov}). This concludes the first part of the 
	proof.
	
	For the second part of the proof, let us first consider a feasible solution 
	$\vvec$ for the original dual
	LP~\eqref{eq:LP-dual}. Then, the choice $\tvec = \omega + \gamma \Psi 
	\vvec$ satisfies the second dual constraint by 
	definition. The first constraint can be verified by writing
	\begin{align*}
		E\vvec - \Phi\theta = E\vvec - r - \gamma P\vvec \ge 0,
	\end{align*}
	where we used the choice of $\tvec$ in the first step and the feasibility 
	of $\vvec$ for the original LP in the second 
	step. Conversely, supposing that $(\tvec,\vvec)\in 
	\mathcal{M}^{D}_{\Phim}$, we note that
	\[
	E\vvec - r - \gamma P\vvec = E\vvec - \Phi\tvec \ge 0,
	\]
	which implies the feasibility of $\vvec$ in the LP~\ref{eq:LP-dual}. 
	Optimality of $\vvec^*$ for both LPs follows from 
	the fact that their objectives are identical, and the standard fact that 
	$\vvec^*$ is an optimal solution of the dual 
	LP~\eqref{eq:LP-dual} (cf.~Theorem~6.2.2 in \citealp{puterman2014markov}).
\end{proof}

\newpage
\section{Missing proofs of Section~\ref{sec:analysis}}
In this section, we provide performance guarantees for 
Algorithm~\ref{alg:main1} in terms of the expected suboptimality of the output 
policy $\pi_{J}$, and in particular prove the lemmas provided in Section~\ref{sec:analysis} in the main text. Auxiliary 
lemmas and technical results for proving some of these are included in Appendix~\ref{app:aux}.

\subsection{Properties of the Dynamic Duality Gap}
We first prove our claims regarding the dynamic duality gap introduced in 
Section~\ref{sec:analysis} of the main text.
First, we relate the gap to the expected suboptimality (in terms of return)
of $\pi_{J}$ against a comparator policy $\pi^{*}$ in Appendix~\ref{appx:gap_to_subopt}. 
Next, we relate the dynamic duality gap to the average regret of each player in Appendix~\ref{appx:gap_to_reg2}. 

\subsubsection{Proof of 
Lemma~\ref{lem:gap_to_subopt}}\label{appx:gap_to_subopt}
By definition of the the dynamic duality gap, we have that
\[
\gap_T\pa{\Phim\transpose\muvec^{\pi^{*}},\pivec^{*},\{\tvec^{\pi_{t}}\}_{t=1}^{T}}
= \frac{1}{T}\sum_{t=1}^{T}
f(\Phim\transpose\muvec^{\pi^{*}},\pivec^{*};\tvec_{t}) - 
	f(\lvec_{t},\pivec_{t};\tvec^{\pi_{t}}).
\]
Considering the first term, we see that
\begin{align*}
	f(\Phim\transpose\muvec^{\pi^{*}},\pivec^{*};\tvec_{t})
	&= \iprod{\Phim\transpose\muvec^{\pi^{*}}}{\rtheta} + 
	\iprod{\tvec_{t}}{\Phim\transpose\muvec_{\lvec^{*},\pivec^{*}} - 
		\Phim\transpose\muvec^{\pi^{*}}}\\
	&\overset{(a)}{=} \iprod{\muvec^{\pi^{*}}}{\rvec} + 
	\iprod{\tvec_{t}}{\Phim\transpose\muvec_{\lvec^{*},\pivec^{*}} - 
		\Phim\transpose\muvec^{\pi^{*}}}\\
	&\overset{(b)}{=} \iprod{\muvec^{\pi^{*}}}{\rvec},
\end{align*}
where we have used $(a)$ the linear MDP property (\cref{def:linMDP}) and $(b)$ the following relation:
\begin{align*}
	\mu_{\lvec^{*},\pivec^{*}}(x,a) 
	&= \pi^{*}(a|x)\Bigl[(1-\df)\nu_0(x) + 
	\df\iprod{\psivec(x)}{\Phim\transpose\muvec^{\pi^{*}}}\Bigr]\\
	&= \pi^{*}(a|x)\Bigl[(1-\df)\nu_0(x) + 
	\df\sum_{x',a'}p\pa{x|x',a'}\mu^{\pi^{*}}\pa{x',a'}\Bigr] = 
	\mu^{\pi^{*}}(x,a).
\end{align*}
Now for the second term, we have
\begin{align*}
	f(\lvec_{t},\pivec_{t};\tvec^{\pi_{t}})
	&= (1-\df)\iprod{\nuvec_0}{\vvec_{\tvec^{\pi_{t}},\pivec_{t}}} 
	+ \iprod{\lvec_{t}}{\rtheta + \gamma 
		\Psim\vvec_{\tvec^{\pi_{t}},\pivec_{t}} 
		- \tvec^{\pi_{t}}}\\
	&= \iprod{\muvec^{\pivec_{t}}}{\rvec}
	+ \iprod{\lvec_{t}}{\rtheta + \gamma 
		\Psim\vvec^{\pi_{t}} - \tvec^{\pi_{t}}}\\
	&= \iprod{\muvec^{\pivec_{t}}}{\rvec},
\end{align*}
where we have used the Bellman equations $\qvec^{\pi_t} = \Phim\tvec^\pi_t = \rvec + \gamma \Pm 
\vvec^{\pi_t} = \Phim\pa{\rtheta + \gamma \Psim \vvec^{\pi_t}}$, which together with the fact that
$\Phim$ is full rank implies that $ \tvec^{\pi_{t}} = \rtheta + \gamma 	\Psim\vvec^{\pi_{t}}$.
Substituting the above expressions for  
$f(\Phim\transpose\muvec^{\pi^{*}},\pivec^{*};\tvec_{t})$ and 
$f(\lvec_{t},\pivec_{t};\tvec^{\pi_{t}})$ in the dynamic duality gap and noting 
that $\pi_{J}$ is such that 
$\frac{1}{T}\sum_{t=1}^{T}\iprod{\muvec^{\pi_{t}}}{\rvec} = 
\EEJ{\iprod{\muvec^{\pi_{J}}}{\rvec}}$ we get
\begin{equation*}
	\gap_T\pa{\Phim\transpose\muvec^{\pi^{*}},\pivec^{*},\{\tvec^{\pi_{t}}\}_{t=1}^{T}}
		= \EEJ{\iprod{\muvec^{\pi^{*}} - \muvec^{\pivec_{J}}}{\rvec}}.
\end{equation*}
This completes the proof.
\hfill\qed

\subsubsection{Proof of Lemma~\ref{lem:gap_to_reg2}}\label{appx:gap_to_reg2}
Recall that for any comparator points 
$\pa{\lvec^{*},\pivec^{*};\{\tvec_{t}^{*}\}_{t=1}^{T}}$, the dynamic duality 
gap is defined as
\[
\gap_T\pa{\lvec^{*},\pivec^{*};\{\tvec_{t}^{*}\}_{t=1}^{T}}	= 
\frac{1}{T}\sum_{t=1}^{T} 
\pa{f(\lvec^{*},\pivec^{*};\tvec_{t}) - 
	f(\lvec_{t},\pivec_{t};\tvec^{*}_{t})}.
\]
Then, by adding and subtracting some terms we express the dynamic duality gap 
in terms of the average loss of each player with respect to the objective 
$f(\lvec,\pivec;\tvec)$. This gives
\begin{align}
	\gap_T(\lvec^{*},\pivec^{*},\tvec^{*}_{1:T})
	\nonumber&= \frac{1}{T}\sum_{t=1}^{T}
	f(\lvec^{*},\pivec^{*};\tvec_{t}) - 
		f(\lvec^{*},\pivec_{t};\tvec_{t})\\
	\nonumber&\quad
	+ \frac{1}{T}\sum_{t=1}^{T}
	f(\lvec^{*},\pivec_{t};\tvec_{t}) - 
		f(\lvec_{t},\pivec_{t};\tvec_{t})\\
	\label{lem:eq:gap}&\quad
	+ \frac{1}{T}\sum_{t=1}^{T}
	f(\lvec_{t},\pivec_{t};\tvec_{t}) - 
		f(\lvec_{t},\pivec_{t};\tvec^{*}_{t}).
\end{align}
Consider the first set of terms from the above expression. By definition of $f$ 
in Equation~\eqref{eq:red_lag}, we immediately obtain the \emph{instantaneous} 
regret of the $\pi$-player as
\begin{align*}
	f(\lvec^{*},\pivec^{*};\tvec_{t}) - 
	f(\lvec^{*},\pivec_{t};\tvec_{t})
	&=
	\iprod{\tvec_{t}}{\Phim\transpose\muvec_{\lvec^{*},\pivec^{*}} - 
		\Phim\transpose\muvec_{\lvec^{*},\pivec_{t}}}\\
	&=
	\sum_{x}\nu^{*}(x)\sum_{a}\pa{\pi^{*}(a|x) - \pi_{t}(a|x)}q_{t}(x,a)\\
	&=
	\sum_{x}\nu^{*}(x)\sum_{a}\pa{\pi^{*}(a|x) - \pi_{t}(a|x)}q_{t}(x,a),
\end{align*}
where $\nu^{*}(x) = (1-\df)\nu_0(x) + \df\iprod{\psivec(x)}{\lvec^{*}}$. For 
the regret of the $\lvec$ and $\tvec$-players, notice that we can express 
the estimator $\widehat{f}$ in terms of the objective $f$ as follows:
\begin{align*}
	\widehat{f}(\lvec,\pivec;\tvec)
	&= (1-\df)\iprod{\nuvec_0}{\vvec_{\tvec,\pivec}} 
	+ \iprod{\lvec}{\rtheta + \gamma \hPsim\vvec_{\tvec,\pivec} - \tvec}\\
	&= f(\lvec,\pivec;\tvec) + \df\iprod{\lvec}{\hPsim\vvec_{\tvec,\pivec} - 
		\Psim\vvec_{\tvec,\pivec}}.
\end{align*}
Taking advantage of this relation, we now consider the last two set of terms in 
Equation~\eqref{lem:eq:gap}. Indeed, for the second set of terms in the 
equation, we write
\begin{align*}
	&f(\lvec^{*},\pivec_{t};\tvec_{t}) - 
	f(\lvec_{t},\pivec_{t};\tvec_{t})\\
	&\quad\quad=
	\widehat{f}(\lvec^{*},\pivec_{t};\tvec_{t}) -
	\widehat{f}(\lvec_{t},\pivec_{t};\tvec_{t})
	- \df\iprod{\lvec^{*}}{\hPsim\vvec_{\tvec_{t},\pivec_{t}} - 
		\Psim\vvec_{\tvec_{t},\pivec_{t}}}
	+ \df\iprod{\lvec_{t}}{\hPsim\vvec_{\tvec_{t},\pivec_{t}} - 
		\Psim\vvec_{\tvec_{t},\pivec_{t}}}\\
	&\quad\quad=
	\iprod{\lvec^{*} - \lvec_{t}}
	{\rtheta + \gamma \hPsim\vvec_{\tvec_{t},\pivec_{t}} - \tvec_{t}}
	- \df\iprod{\lvec^{*}}{\hPsim\vvec_{\tvec_{t},\pivec_{t}} - 
		\Psim\vvec_{\tvec_{t},\pivec_{t}}}
	+ \df\iprod{\lvec_{t}}{\hPsim\vvec_{\tvec_{t},\pivec_{t}} - 
		\Psim\vvec_{\tvec_{t},\pivec_{t}}},
\end{align*}
Notice that the last equality follows directly from definition of $\wh{f}$. 
Along these lines, we can also express the last set of terms in 
Equation~\eqref{lem:eq:gap} as follows:
\begin{align*}
	&f(\lvec_{t},\pivec_{t};\tvec_{t}) - 
	f(\lvec_{t},\pivec_{t};\tvec^{*}_{t})\\
	&\quad\quad=
	\widehat{f}(\lvec_{t},\pivec_{t};\tvec_{t}) - 
	\widehat{f}(\lvec_{t},\pivec_{t};\tvec^{*}_{t})
	- \df\iprod{\lvec_{t}}{\hPsim\vvec_{\tvec_{t},\pivec_{t}} - 
		\Psim\vvec_{\tvec_{t},\pivec_{t}}}
	+ \df\iprod{\lvec_{t}}{\hPsim\vvec_{\tvec_{t},\pivec_{t}} - 
		\Psim\vvec_{\tvec^{*}_{t},\pivec_{t}}}\\
	&\quad\quad=
	\iprod{\tvec_{t} - 
		\tvec^{*}_{t}}{\Phim\transpose\hmuvec_{\lvec_{t},\pivec_{t}} 
		- \lvec_{t}}
	- \df\iprod{\lvec_{t}}{\hPsim\vvec_{\tvec_{t},\pivec_{t}} - 
		\Psim\vvec_{\tvec_{t},\pivec_{t}}}
	+ \df\iprod{\lvec_{t}}{\hPsim\vvec_{\tvec_{t},\pivec_{t}} - 
		\Psim\vvec_{\tvec^{*}_{t},\pivec_{t}}},
\end{align*}
Plugging the above derivations in the dynamic duality gap, we have that
\begin{align}
	\gap_T(\lvec^{*},\pivec^{*},\tvec^{*}_{1:T})
	\nonumber&= \frac{1}{T}\sum_{t=1}^{T}
	\sum_{x}\nu^{*}(x)\sum_{a}\pa{\pi^{*}(a|x) - 
			\pi_{t}(a|x)}q_{t}(x,a)\\
	\nonumber&\quad
	+ \frac{1}{T}\sum_{t=1}^{T}
	\iprod{\lvec^{*} - \lvec_{t}}{\rtheta + \gamma
			\hPsim\vvec_{\tvec_{t},\pivec_{t}} - \tvec_{t}}\\
	\nonumber&\quad
	+ \frac{1}{T}\sum_{t=1}^{T}
	\iprod{\tvec_{t} - 	
			\tvec^{*}_{t}}{\Phim\transpose\hmuvec_{\lvec_{t},\pivec_{t}} - 
			\lvec_{t}}\\
	\nonumber&\quad
	+ \frac{\df}{T}\sum_{t=1}^{T}
	\iprod{\lvec^{*}}{\Psim\vvec_{\tvec_{t},\pivec_{t}} - 
			\hPsim\vvec_{\tvec_{t},\pivec_{t}}}
	+ \frac{\df}{T}\sum_{t=1}^{T}
	\iprod{\lvec_{t}}{\hPsim\vvec_{\tvec_{t}^{*},\pivec_{t}} - 
			\Psim\vvec_{\tvec_{t}^{*},\pivec_{t}}}.
\end{align}
This matches the claim of the lemma, thus completing the proof.
\hfill\qed

\subsection{Bounding the Regret Terms}\label{appx:REG_proof}
In this section we provide the proofs of the our claims made in the main text about the regret of each 
player---precisely, Lemmas~\ref{lem:piregret}--\ref{lem:thetaregret}.
\subsubsection{Proof of Lemma~\ref{lem:piregret}}\label{appx:piregret}
Consider the regret of the $\pi$-player introduced in the main text as,
\begin{align*}
	\regret_{T}\pa{\pivec^{*}}
	&=\sum_{t=1}^{T}
	{\sum_{x}\nu^{*}(x)\sum_{a}\pa{\pi^{*}(a|x) - 
			\pi_{t}(a|x)}q_{t}(x,a)}\\
	&\overset{(a)}{\leq}
	\frac{\sum_{x}\nu^{*}(x)\DDKL{\pi^{*}\pa{\cdot|x}}{\pi_{1}\pa{\cdot|x}}}{\alpha}
	+ \frac{\alpha 
		TR^{2}D_{\tvec}^{2}}{2}\\
	&\overset{(b)}{\leq}
	\frac{\log A}{\alpha } + \frac{\alpha T	R^{2}D_{\tvec}^{2}}{2}.
\end{align*}
We have used $(a)$ the standard Mirror descent analysis of softmax policy 
iterates recalled in Lemma~\ref{lem:MD} for completeness, and $(b)$ the fact 
that $\pi_{1}$ is a uniform policy and $\nuvec^{*}\in\Delta_{\X}$. Dividing the 
above expression by $T$ completes the proof.\hfill\qed

\subsubsection{Proof of Lemma~\ref{lem:lambdaregret}}\label{appx:lambdaregret}
Recall the total regret of the $\lvec$-player against any fixed comparator 
$\lvec^{*}\in\Rn^{d}$ is given as
\[
	\regret_{T}\pa{\lvec^{*}}
	=
	\sum_{t=1}^{T}
	{\ip{\lvec^{*} - \lvec_{t}}{\rtheta + \gamma
			\hPsim\vvec_{\tvec_{t},\pivec_{t}} - \tvec_{t}}}.
\]
Since the feature-occupancy updates of Algorithm~\ref{alg:main1} 
simply implements a version of the composite-objective mirror descent scheme 
due to \cite{duchi2010composite} we apply the standard analysis of this method (recalled as Lemma~\ref{lem:COGDA} in 
Appendix~\ref{appx:REGlem_proof}) to bound the 
instantaneous regret as
\begin{align*}
	&\ip{\lvec^{*} - \lvec_{t}}{\rtheta + \gamma
		\hPsim\vvec_{\tvec_{t},\pivec_{t}} - \tvec_{t}}\\
	&\qquad\leq
	\frac{\norm{\lvec_{t} - \lvec^{*}}^{2}_{\Lm^{-1}} - \norm{\lvec_{t+1} - 
			\lvec^{*}}^{2}_{\Lm^{-1}}}{2\eta}
	+ \frac{\eta}{2}\norm{\Lm\gl}^{2}_{\Lm^{-1}}
	+ \frac{\varrho}{2}\norm{\lvec^{*}}^{2}_{\Lm^{-1}}
	- \frac{\varrho}{2}\norm{\lvec_{t+1}}^{2}_{\Lm^{-1}}.
\end{align*}
Then, taking the sum for $t=1,\cdots,T$, evaluating the 	
telescoping sums and upper-bounding some negative terms by zero yields the 
expression
\begin{align*}
	&\sum_{t=1}^{T}\ip{\lvec^{*} - \lvec_{t}}{\rtheta + \gamma
		\hPsim\vvec_{\tvec_{t},\pivec_{t}} - \tvec_{t}}\\
	&\qquad\leq
	\frac{\norm{\lvec_{1} - \lvec^{*}}^{2}_{\Lm^{-1}}}{2\eta}
	+ \frac{\eta}{2}\sum_{t=1}^{T}\norm{\Lm\gl}^{2}_{\Lm^{-1}}
	+ \frac{\varrho T}{2}\norm{\lvec^{*}}^{2}_{\Lm^{-1}}
	- \frac{\varrho}{2}\sum_{t=1}^{T}\norm{\lvec_{t}}^{2}_{\Lm^{-1}}
	+\frac{\varrho}{2}\norm{\lvec_{1}}^{2}_{\Lm^{-1}}\\
	&\qquad=
	\pa{\frac{1}{2\eta} + \frac{\varrho 
			T}{2}}\norm{\lvec^{*}}^{2}_{\Lm^{-1}}
	+ \frac{\eta}{2}\sum_{t=1}^{T}\norm{\Lm\gl}^{2}_{\Lm^{-1}}
	- \frac{\varrho}{2}\sum_{t=1}^{T}\norm{\lvec_{t}}^{2}_{\Lm^{-1}}.
\end{align*}
In the equality, we have used that $\lvec_{1}=\vec{0}$. Dividing the resulting 
term by $T$ gives the following bound on the average regret:
\[
	\frac{1}{T}\regret_{T}\pa{\lvec^{*}}
	\leq
	\pa{\frac{1}{2T\eta} + \frac{\varrho 
			}{2}}\norm{\lvec^{*}}^{2}_{\Lm^{-1}}
	+ \frac{\eta}{2T}\sum_{t=1}^{T}\norm{\Lm\gl}^{2}_{\Lm^{-1}}
	- \frac{\varrho}{2T}\sum_{t=1}^{T}\norm{\lvec_{t}}^{2}_{\Lm^{-1}}.
\]
The proof is completed by applying Lemma~\ref{lem:gnorm} to bound the norm of the gradients and plugging the 
result into the bound above. \hfill\qed

\subsubsection{Proof of Lemma~\ref{lem:thetaregret}}\label{appx:thetaregret}
For the regret of the $\tvec$-player, first note that for any policy $\pi$ with 
corresponding 
state-action value function weights $\tvec^{\pi}=\rtheta + 
\df\Psim\vvec^{\pi}$, we have
\begin{align*}
	\twonorm{\tvec^{\pi}}
	&= \twonorm{\rtheta + \df\Psim\vvec^{\pi}}\leq \twonorm{\rtheta} + \df\twonorm{\Psim\vvec^{\pi}}\leq \sqrt{d} + 
\frac{\df\sqrt{d}}{\pa{1-\df}} = 
	\frac{\sqrt{d}}{\pa{1-\df}},
\end{align*}
where we have used the triangle inequality in the second line. The last inequality 
uses Definition~\ref{def:linMDP} and the fact that 
$\infnorm{\vvec^{\pi}}\leq\frac{1}{\pa{1-\df}}$ since the rewards are bounded 
in $[0,1]$. Thanks to this bound, we can ensure that $\tvec^{*}_{t}=\tvec^{\pi_{t}}\in\bb{d}{D_{\tvec}}$ 
holds with the choice  $D_{\tvec}=\sqrt{d}/\pa{1-\df}$ as required by the lemma.Therefore, by construction 
of value-parameter updates in Algorithm~\ref{alg:main1}, we have
\[
\iprod{\tvec_{t} - 
	\tvec^{*}_{t}}{\Phim\transpose\hmuvec_{\lvec_{t},\pivec_{t}} - 
	\lvec_{t}}\leq 0 \qquad\text{for } t=1,\cdots,T.
\]
This concludes the proof.\hfill\qed

\subsection{Bounding the gap-estimation error}\label{appx:GER_proof}
In this section, we provide the proof of Lemma~\ref{lem:GER} which 
bounds the gap-estimation error defined for an arbitrary 
comparator sequence 
$\pa{\lvec^{*},\pi_{t},\tvec_{t}^{*}}\in\Rn^{d}\times\Pi\pa{D_{\pi}}\times\bb{d}{D_{\tvec}}$
 for 
$t=1,\dots,T$ as,
\begin{equation*}
	\GER = \sum_{t=1}^{T}
	\biprod{\lvec^{*}}{\bpa{\Psim - \hPsim}\vvec_{\tvec_{t},\pi_{t}}}
	+ \sum_{t=1}^{T}
	\biprod{\lvec_{t}}{\bpa{\hPsim - \Psim}\vvec_{\tvec_{t}^{*},\pivec_{t}}}.
\end{equation*}
We control the above term with the now-classic techniques developed by 
\citet{jin2020provably} for bounding model-estimation errors for linear MDPs.
These results also make heavy use of self-normalized tail inequalities as popularized by
\citet{abbasi2011improved} (see also \citealp{lattimore2020bandit}). To make this clear, 
we first note that, for any $\lvec\in\Rn^{d}$, 
$\vvec\in\Rn^{X}$, and $\xi>0$, 
\[
\iprod{\lvec}{\pa{\hPsim - \Psim}\vvec}
\overset{(a)}{\leq}
\norm{\lvec}_{\Lm^{-1}}\norm{\Lm\pa{\hPsim - \Psim}\vvec}_{\Lm^{-1}}
\overset{(b)}{\leq}
\frac{\norm{\lvec}_{\Lm^{-1}}^{2}}{2 T\xi} + 
\frac{T\xi}{2}\norm{\Lm\pa{\hPsim - \Psim}\vvec}_{\Lm^{-1}}^{2}.
\]
Here, we have first used $(a)$ the Cauchy--Schwarz inequality, and $(b)$ the inequality of arithmetic and geometric 
means. 
Using this expression, we can upper-bound the gap estimation error as
\begin{align}\label{eq:GER1}
	\GER
	\nonumber&\leq
	\frac{\norm{\lvec^{*}}_{\Lm^{-1}}^{2}}{2 \xi}
	+
	\sum_{t=1}^{T}\frac{\norm{\lvec_{t}}_{\Lm^{-1}}^{2}}{2 T\xi}\\
	&\quad+ 
	\frac{T\xi}{2}\sum_{t=1}^{T}\norm{\Lm\pa{\hPsim - 
			\Psim}\vvec_{\tvec_{t},\pi_{t}}}_{\Lm^{-1}}^{2}
	+ 
	\frac{T\xi}{2}\sum_{t=1}^{T}\norm{\Lm\pa{\hPsim - 
			\Psim}\vvec^{\pi_{t}}}_{\Lm^{-1}}^{2}.
\end{align}
To control the last two terms in the bound, we employ two main lemmas stated 
below.
\begin{lemma}\label{lem:LSQerror1}
	Let $\vvec \in [-B,B]^{X}$. With probability at least $1-\delta$, we have 
	that:
	\[
	\norm{\Lm\pa{\hPsim - \Psim}\vvec}_{\Lm^{-1}}
	\leq
	\frac{2B}{\sqrt{n}}\sqrt{d\log\Biggl(1 + \frac{R^{2}}{d\beta}\Biggr) + 
		2\log\frac{1}{\delta}}
	+
	B\sqrt{d\beta}.
	\]
\end{lemma}
\begin{lemma}\label{lem:LSQerror2}
	Consider the function class,
	\[
	\VV = \Bigl\{\vvec_{\pivec,\tvec}:\X\rightarrow 
	[-RD_{\tvec},RD_{\tvec}] \Big| 
	\pivec\in\Pi\pa{D_{\pivec}}, \tvec\in\bb{d}{D_{\tvec}}\Bigr\},
	\]
	Let $D_{\pivec} = \alpha TD_{\tvec}$ so that 
	$\vvec_{\theta_{t},\pi_{t}}\in\VV$. For any $\epsilon\in\pa{0,1}$, with 
	probability at least $1-\delta$,
	\begin{align*}
		&\norm{\Lm\pa{\hPsim - \Psim}\vvec_{\theta_{t},\pi_{t}}}_{\Lm^{-1}}\\
		&\qquad\qquad\leq
		\frac{2RD_{\tvec}}{\sqrt{n}}\sqrt{d\log\Biggl(1 + 
			\frac{R^{2}}{d\beta}\Biggr) + 
			4d\log\pa{1+\frac{4\alpha TR^{2}D_{\tvec}^{2}}{\epsilon}} + 
			2\log\frac{1}{\delta}}\\
		&\quad\qquad\qquad
		+
		RD_{\tvec}\sqrt{d\beta}
		+
		\pa{\sqrt{\beta} + 1}\epsilon\sqrt{d}.
	\end{align*}
\end{lemma}
The rather tedious but otherwise standard proofs of the above lemmas are given in 
Appendix~\ref{appx:GERlem_proof}. Now, taking into account the fact that for 
$D_{\tvec}$ large enough $\vvec^{\pi_{t}}\in\VV$ yields 
Corollary~\ref{cor:LSQerror3} below.
\begin{corollary}\label{cor:LSQerror3}
	In the linear MDP setting described in Definition~\ref{def:linMDP}, notice 
	that 
	$\vvec^{\pi_{t}}=\vvec_{\tvec^{\pivec_{t}},\pivec_{t}}$ with 	
	$\tvec^{\pivec_{t}} = \rtheta + 
	\df\Psim\vvec_{\tvec^{\pi_{t}},\pivec_{t}}$. Furthermore, with 
	$RD_{\tvec}=R\sqrt{d}/(1-\df)\geq\infnorm{\vvec^{\pi_{t}}}$ and $D_{\pivec} 
	= \alpha TD_{\tvec}$ we have that $\vvec^{\pi_{t}}\in\VV$.
	Therefore, for all $\epsilon>0$ with probability at least $1-\delta$, the 
	following holds:
	\begin{align*}
		&\norm{\Lm\pa{\hPsim - \Psim}\vvec^{\pi_{t}}}_{\Lm^{-1}}\\
		&\qquad\leq
		\frac{2\sqrt{d}}{\sqrt{n}\pa{1-\df}}\sqrt{d\log\Biggl(1 + 
			\frac{R^{2}}{d\beta}\Biggr) + 
			4d\log\pa{1+\frac{4\alpha TR^{2}d}{\epsilon\sq{1-\df}}} + 
			2\log\frac{1}{\delta}}
		\\&\quad\qquad
		+
		\frac{d\sqrt{\beta}}{\pa{1-\df}}
		+
		\pa{\sqrt{\beta} + 1}\epsilon\sqrt{d}.
	\end{align*}
\end{corollary}
In the following, we apply these results to bound the last two terms 
in the right-hand side of Equation~\eqref{eq:GER1}. Precisely, using 
$D_{\tvec}=\sqrt{d}/(1-\df)$, 
$\alpha=\sqrt{2\log A/R^{2}D_{\tvec}^{2}T} = \sqrt{2\sq{1-\df}\log 
A/R^{2}dT}$ (which follows from optimizing the regret of the $\pi$-player in 
Lemma~\ref{lem:piregret}), as well as $\epsilon = 4\alpha R^{2}d/\sq{1-\df} =  
\frac{\sqrt{32R^{2}d\log A}}{\pa{1-\df}\sqrt{T}}$ and $\beta = 
R^{2}/dT$ we have that in any round $t$, with probability at least $1-\delta$,
\begin{equation*}
	\norm{\Lm\pa{\hPsim - 
			\Psim}\vvec_{\tvec_{t},\pi_{t}}}_{\Lm^{-1}}
	\leq
	\sqrt{\frac{20d^{2}\log\pa{2T/\delta}}{n\sq{1-\df}}}
	+ \sqrt{\frac{R^{2}d}{T\sq{1-\df}}}
	+ \sqrt{\frac{R^{4}d\log A}{T^{2}}}
	+ \sqrt{\frac{32R^{2}d^{2}\log A}{T\sq{1-\df}}}.
\end{equation*}
Likewise,
\begin{equation*}
	\norm{\Lm\pa{\hPsim - 
			\Psim}\vvec^{\pi_{t}}}_{\Lm^{-1}}
	\leq
	\sqrt{\frac{20d^{2}\log\pa{2T/\delta}}{n\sq{1-\df}}}
	+ \sqrt{\frac{R^{2}d}{T\sq{1-\df}}}
	+ \sqrt{\frac{R^{4}d\log A}{T^{2}}}
	+ \sqrt{\frac{32R^{2}d^{2}\log A}{T\sq{1-\df}}}.
\end{equation*}
Then plugging in the 
above bounds in Equation~\eqref{eq:GER1} with 
$T\ge\frac{2R^2 n \log A}{\log(1/\delta)}$, 
it follows that 
for $D_{\hPsim} = \sqrt{\frac{320d^{2}\log\pa{2T/\delta}}{n\sq{1-\df}}}$,
\[
	\GER
	\leq
	\frac{\norm{\lvec^{*}}_{\Lm^{-1}}^{2}}{2 \xi}
	+
	\sum_{t=1}^{T}\frac{\norm{\lvec_{t}}_{\Lm^{-1}}^{2}}{2 T\xi}
	+ T^{2}\xi D_{\hPsim}^{2},
\]
with probability at least $1-\delta$, thus proving the claim. \hfill\qed

\subsection{Full proof of Theorem~\ref{thm:main1}}\label{appx:main_proof}
To control the expected suboptimality of the output policy $\pi_{J}$ of 
Algorithm~\ref{alg:main1}, we study the repective regret and gap-estimation 
error at the selected comparator points. Precisely, 
combining Lemma~\ref{lem:gap_to_subopt} and \ref{lem:gap_to_reg2} when 
$\pa{\lvec^{*},\pivec^{*},\tvec^{*}_{1:T}} = 
\pa{\Phim\transpose\muvec^{\pi^{*}}, \pivec^{*}, \tvec^{\pi_{t}}} \in 
\Rn^{d}\times 
\Pi\pa{D_{\pi}}\times \bb{d}{D_{\theta}}$, we have that,
\begin{align}\label{eq:subopt}
	\EEJ{\iprod{\muvec^{\pi^{*}} - \muvec^{\pi_{J}}}{\rvec}}
	&=
	\frac{1}{T}\regret_{T}\pa{\pivec^{*}}
	+ \frac{1}{T}\regret_{T}\pa{\lvec^{\pi^{*}}}
	+ \frac{1}{T}\regret_{T}\pa{\tvec^{*}_{1:T}}
	+ \frac{\df}{T}\GER.
\end{align}
where,
\begin{align*}
	\regret_{T}\pa{\pivec^{*}}
	&=
	\sum_{t=1}^{T}
	{\sum_{x}\nu^{*}(x)\sum_{a}\pa{\pi^{*}(a|x) - 
			\pi_{t}(a|x)}q_{t}(x,a)},&\\
	\regret_{T}\pa{\lvec^{\pi^{*}}}
	&=
	\sum_{t=1}^{T}
	{\ip{\lvec^{\pi^{*}} - \lvec_{t}}{\rtheta + \gamma
			\hPsim\vvec_{\tvec_{t},\pivec_{t}} - \tvec_{t}}},\\
	\regret_{T}\pa{\tvec^{*}_{1:T}}
	&=
	\sum_{t=1}^{T}
	{\ip{\tvec_{t} - 
			\tvec^{\pi_{t}}_{t}}{\Phim\transpose\hmuvec_{\lvec_{t},\pivec_{t}} 
			- \lvec_{t}}}\\
	\GER
	&= \sum_{t=1}^{T}
	\biprod{\lvec^{\pi^{*}}}{\bpa{\Psim - \hPsim}\vvec_{\tvec_{t},\pi_{t}}}
	+ \sum_{t=1}^{T}
	\biprod{\lvec_{t}}{\bpa{\hPsim - \Psim}\vvec^{\pivec_{t}}}.
\end{align*}
Notice that for this choice of $\lvec^{*}$, by Definition~\ref{def:linMDP} 
$\nu^{*}(x) = (1-\df)\nu_0(x) + 
\df\ip{\psivec(x)}{\muvec^{*}}=\nu^{\pi^{*}}(x)$ is a valid state occupancy 
measure.
Next, introducing the bounds stated in Lemmas~\ref{lem:piregret}--\ref{lem:GER} 
under the required conditions $D_{\tvec} = \sqrt{d}/\pa{1-\df}$, 
$\alpha=\sqrt{2\sq{1-\df}\log 
A/R^{2}dT}$, $D_{\pi} = \alpha TD_{\tvec} = \sqrt{2T\log 
A/R^{2}}$ and $T\ge \frac{2 R^2n \log A}{\log(1/\delta)}$, as well as $\xi\ge0$ 
and 
$D_{\hPsim} = \sqrt{\frac{320d^{2}\log\pa{2T/\delta}}{n\sq{1-\df}}}$ yields,
\begin{align*}
	\EEJ{\iprod{\muvec^{\pi^{*}} - \muvec^{\pi_{J}}}{\rvec}}
	&\leq
	\sqrt{\frac{2R^{2}d\log A}{\sq{1-\df}T}}\\
	&\quad
	+
	\pa{\frac{1}{2\eta T} + \frac{\varrho 
		}{2}}\norm{\lvec^{\pi^{*}}}^{2}_{\Lm^{-1}}
	+ \frac{\eta C}{2}
	- \frac{\varrho}{2T}\sum_{t=1}^{T}\norm{\lvec_{t}}^{2}_{\Lm^{-1}}\\
	&\quad
	+
	\frac{\df}{2T\xi} \pa{{\norm{\lvec^{\pi^{*}}}_{\Lm^{-1}}^{2}} + 
		\frac{1}{T} \sum_{t=1}^{T}
		{\norm{\lvec_{t}}_{\Lm^{-1}}^{2}}} + \df T\xi D_{\hPsim}^{2},
\end{align*}
with probability at least $1-\delta$, where $C = 6\beta\pa{d + D_{\tvec}^{2}} + 
3d\pa{1+RD_{\tvec}}^{2} + 3\df^{2} dR^{2}D_{\tvec}^{2}$. Rearranging the bound 
and selecting $\varrho=\df/\xi T$ to eliminate the (potentially large) norm of 
the iterates, we obtain
\begin{align*}
	\EEJ{\iprod{\muvec^{\pi^{*}} - \muvec^{\pi_{J}}}{\rvec}}
	&\leq
	\sqrt{\frac{d\log\pa{1/\delta}}{n\sq{1-\df}}}
	+
	\pa{\frac{1}{2\eta T} + \frac{\varrho}{2} + \frac{\df}{2\xi T}}
	\norm{\lvec^{\pi^{*}}}^{2}_{\Lm^{-1}}
	+ \frac{\eta C}{2}\\
	&\quad
	+
	\pa{\frac{\df}{\xi T} - \varrho}\frac{1}{2T}\sum_{t=1}^{T}
	{\norm{\lvec_{t}}_{\Lm^{-1}}^{2}} + \df T\xi D_{\hPsim}^{2}\\
	&=
	\sqrt{\frac{d\log\pa{1/\delta}}{n\sq{1-\df}}}
	+
	\pa{\frac{1}{2\eta T} + \frac{\df}{\xi T}}
	\norm{\lvec^{\pi^{*}}}^{2}_{\Lm^{-1}}
	+ \frac{\eta C}{2} + \df T\xi D_{\hPsim}^{2}.
\end{align*}
Furthermore, choosing $\xi=1/TD_{\hPsim}$ i.e $\varrho=\df D_{\hPsim}$, we 
further 
simplify the above bound on the regret in terms of the optimization error  
arising from the policy and feature occupancy updates as,
\begin{align*}
	\EEJ{\iprod{\muvec^{\pi^{*}} - \muvec^{\pi_{J}}}{\rvec}}
	&\leq
	\sqrt{\frac{d\log\pa{1/\delta}}{n\sq{1-\df}}}
	+
	\frac{1}{2\eta T}\norm{\lvec^{\pi^{*}}}^{2}_{\Lm^{-1}}
	+ \frac{\eta C}{2}
	+
	\df\pa{\norm{\lvec^{\pi^{*}}}^{2}_{\Lm^{-1}} + 1} D_{\hPsim}.
\end{align*}
Moving our attention to our earlier bound on the norm of $\gl$,
\[
	C = 6\beta\pa{d + D_{\tvec}^{2}} + 
	3d\pa{1+RD_{\tvec}}^{2} + 3\df^{2} dR^{2}D_{\tvec}^{2}\leq 
	\frac{27R^{2}d^{2}}{\sq{1-\df}}.
\]
The inequality follows from our earlier choice of $\beta=R^{2}/dT$ and that 
$T\geq1/d^{2}$. Plugging the values of $C$ and $D_{\hPsim}$ in the bound, then 
choosing $\eta = 
\sqrt{\frac{\sq{1-\df}}{27R^{2}d^{2}T}}$ and using the condition $T\ge \frac{2 R^2n \log 
A}{\log(1/\delta)}$, we have that with probability at least 
$1-\delta$,
\begin{align*}
	\EEJ{\iprod{\muvec^{\pi^{*}} - \muvec^{\pi_{J}}}{\rvec}}
	&\leq
	\sqrt{\frac{d\log\pa{1/\delta}}{n\sq{1-\df}}}+
	\pa{\norm{\lvec^{\pi^{*}}}^{2}_{\Lm^{-1}}+1}\sqrt{\frac{27d^{2}\log\pa{1/\delta}}{8n\log
	 A\sq{1-\df}}}\\
	&\quad
	+
	\df\pa{\norm{\lvec^{\pi^{*}}}^{2}_{\Lm^{-1}} + 
	1}\sqrt{\frac{320d^{2}\log\pa{2T/\delta}}{n\sq{1-\df}}}\\
	&= 
	\mathcal{O}\pa{\frac{\norm{\lvec^{\pi^{*}}}^{2}_{\Lm^{-1}}+1}{\pa{1-\df}} 
	\sqrt{\frac{d^{2}\log\pa{2T/\delta}}{n}}}.
\end{align*}
This completes the proof.\hfill\qed

\newpage
\section{Missing proofs of 	
Section~\ref{appx:REG_proof}}\label{appx:REGlem_proof}

\begin{lemma}(cf. Lemma 1 of \citet{duchi2010composite})\label{lem:COGDA}
	Let $\gl = \rtheta + \gamma\hPsim\vvec_{\tvec_{t},\pivec_{t}} - \tvec_{t}$
	Given 
	$\lvec_{1}=\vec{0}$ and $\varrho,\eta>0$ and the sequence of iterates
	$\{\lvec_{t}\}_{t=2}^{T}$ defined for $t=1,\cdots,T$ as:
	\begin{equation}\label{eq:lambda_update1}
		\lvec_{t+1}
		=
		\argmin_{\lvec\in\Rn^{d}}\Bigl\{-\iprod{\lvec}{\gl}
		+ \frac{1}{2\eta}\norm{\lvec - \lvec_{t}}^{2}_{\Lm^{-1}}
		+ \frac{\varrho}{2}\norm{\lvec}^{2}_{\Lm^{-1}}\Bigr\}.
	\end{equation}
	Then, for any $\lvec^{*}\in\Rn^{d}$,
	\begin{align*}
		&\ip{\lvec^{*} - \lvec_{t}}{\rtheta + \gamma
			\hPsim\vvec_{\tvec_{t},\pivec_{t}} - \tvec_{t}}\\
		&\qquad\leq
		\frac{\norm{\lvec_{t} - \lvec^{*}}^{2}_{\Lm^{-1}} - \norm{\lvec_{t+1} - 
				\lvec^{*}}^{2}_{\Lm^{-1}}}{2\eta}
		+ \frac{\eta}{2}\norm{\Lm\gl}^{2}_{\Lm^{-1}}
		+ \frac{\varrho}{2}\norm{\lvec^{*}}^{2}_{\Lm^{-1}}
		- \frac{\varrho}{2}\norm{\lvec_{t+1}}^{2}_{\Lm^{-1}}.
	\end{align*}
\end{lemma}

\begin{proof}
	The proof of Lemma~\ref{lem:COGDA} follows directly from the referenced 
	Lemma 
	from \citet{duchi2010composite}. Consider,
	\begin{align*}
		&\iprod{\lvec^{*} - \lvec_{t}}{\gl}
		+ \frac{\varrho}{2}\norm{\lvec_{t+1}}^{2}_{\Lm^{-1}}
		- \frac{\varrho}{2}\norm{\lvec^{*}}^{2}_{\Lm^{-1}}\\
		&\qquad= \iprod{\lvec_{t+1} - \lvec^{*}}{-\gl + 
			\frac{1}{\eta}\Lm^{-1}\pa{\lvec_{t+1} 
				- \lvec_{t}} + \varrho\Lm^{-1}\lvec_{t+1}}
			+ \iprod{\lvec_{t+1} - \lvec_{t}}{\gl}\\
		&\qquad\quad
		- \iprod{\lvec_{t+1} - \lvec^{*}}{
			\frac{1}{\eta}\Lm^{-1}\pa{\lvec_{t+1} 
				- \lvec_{t}} + \varrho\Lm^{-1}\lvec_{t+1}}
		+ \frac{\varrho}{2}\norm{\lvec_{t+1}}^{2}_{\Lm^{-1}}
		- \frac{\varrho}{2}\norm{\lvec^{*}}^{2}_{\Lm^{-1}}\\
		&\qquad\overset{(a)} {\leq}
		\iprod{\lvec_{t+1} - \lvec_{t}}{\gl}
		- \frac{1}{\eta}
		\iprod{\lvec_{t+1} - \lvec^{*}}{\Lm^{-1}\pa{\lvec_{t+1}- \lvec_{t}}}\\
		&\qquad\quad
		+\varrho\iprod{\lvec^{*}}{\Lm^{-1}\lvec_{t+1}}
		- \frac{\varrho}{2}\norm{\lvec_{t+1}}^{2}_{\Lm^{-1}}
		- \frac{\varrho}{2}\norm{\lvec^{*}}^{2}_{\Lm^{-1}}\\
		&\qquad\overset{(b)} {\leq}
		\iprod{\lvec_{t+1} - \lvec_{t}}{\gl}
		+ \frac{1}{\eta}
		\iprod{\lvec_{t+1} - \lvec^{*}}{\Lm^{-1}\pa{\lvec_{t}- \lvec_{t+1}}}\\
		&\qquad\overset{(c)} {=}
		\iprod{\lvec_{t+1} - \lvec_{t}}{\gl}
		- \frac{1}{2\eta}\norm{\lvec_{t+1} - \lvec_{t}}^{2}_{\Lm^{-1}}
		+ \frac{1}{2\eta}\pa{\norm{\lvec^{*} - \lvec_{t}}^{2}_{\Lm^{-1}} - 
		\norm{\lvec^{*} - \lvec_{t+1}}^{2}_{\Lm^{-1}}}\\
		&\qquad\leq
		\frac{1}{\eta}\sup_{\yvec\in\Rn^{d}}\pa{\iprod{\yvec}{\eta\Lm^{1/2}\gl} 
		- 
		\frac{1}{2}\norm{\yvec}^{2}_{2}}
		+ \frac{1}{2\eta}\pa{\norm{\lvec^{*} - \lvec_{t}}^{2}_{\Lm^{-1}} - 
			\norm{\lvec^{*} - \lvec_{t+1}}^{2}_{\Lm^{-1}}}\\
		&\qquad\overset{(d)} {=}
		\frac{\eta}{2}\norm{\Lm\gl}^{2}_{\Lm^{-1}}
		+ \frac{1}{2\eta}\pa{\norm{\lvec^{*} - \lvec_{t}}^{2}_{\Lm^{-1}} - 
			\norm{\lvec^{*} - \lvec_{t+1}}^{2}_{\Lm^{-1}}}
	\end{align*}
	We have used
	\begin{itemize}
		\item[$(a)$] The first order optimality condition on 
		Equation~\eqref{eq:lambda_update1}:
		
		For any $\lvec\in\Rn^{d}$,
		\[
			\iprod{\lvec_{t+1} - \lvec}{-\gl + 
				\frac{1}{\eta}\Lm^{-1}\pa{\lvec_{t+1} 
					- \lvec_{t}} + \varrho\Lm^{-1}\lvec_{t+1}}\leq 0.
		\]
		\item[$(b)$] The relation:
		\[
			\varrho\iprod{\lvec^{*}}{\Lm^{-1}\lvec_{t+1}}
			- \frac{\varrho}{2}\norm{\lvec_{t+1}}^{2}_{\Lm^{-1}}
			- \frac{\varrho}{2}\norm{\lvec^{*}}^{2}_{\Lm^{-1}}
			= - \frac{\varrho}{2}\norm{\lvec_{t+1} - 
			\lvec^{*}}^{2}_{\Lm^{-1}}\leq 0.
		\]
		\item[$(c)$] By definition of the squared $L^{2}$-norm for vectors 
		$\vec{a} = \Lm^{-1/2}\pa{\lvec_{t+1} - \lvec^{*}}$ and $\vec{b} = 
		\Lm^{-1/2}\pa{\lvec_{t}- \lvec_{t+1}}$:
		\[
		   \iprod{\vec{a}}{\vec{b}} = 
		   \frac{1}{2}\pa{- \sqtwonorm{\vec{b}} + \sqtwonorm{\vec{a}+\vec{b}} - 
		   \sqtwonorm{\vec{a}}}.
		\]
		Note that $\vec{a}$ and $\vec{b}$ are well defined since $\Lm$ is 
		both symmetric and positive definite.
		\item[$(d)$] By definition of the Fenchel conjugate of 
		$\frac{1}{2}\sqtwonorm{\yvec}$ for $\yvec\in\Rn^{d}$.
	\end{itemize}
	Rearranging the terms and plugging in $\gl = \rtheta + 
	\gamma\hPsim\vvec_{\tvec_{t},\pivec_{t}} - \tvec_{t}$ completes the proof.
\end{proof}
Finally, we will use the following result that bounds the gradient norms appearing in the bound 
above.
\begin{lemma}\label{lem:gnorm}
Under the conditions of the linear MDP setting we have that,
	\begin{align*}
		\norm{\Lm\gl}^{2}_{\Lm^{-1}}
		&\leq
		6\beta\pa{d + D_{\tvec}^{2}} + 3d\pa{1+RD_{\tvec}}^{2} + 3\df^{2} 
		dR^{2}D_{\tvec}^{2}.
	\end{align*}
\end{lemma}
\begin{proof}
	Recall that for $t=1,\cdots,T$ $\gl = \rtheta + 
	\df\hPsim\vvec_{\tvec_{t},\pivec_{t}} - \tvec_{t}$. Then,
	\begin{align*}
		&\norm{\Lm\gl}^{2}_{\Lm^{-1}}
		=
		\norm{\Lm\Big[\rtheta + 
			\df\hPsim\vvec_{\tvec_{t},\pivec_{t}} - 
			\tvec_{t}\Big]}^{2}_{\Lm^{-1}}\\
		&\qquad\quad=
		\norm{\beta\pa{\rtheta - \tvec_{t}}
			+ \frac{1}{n}\sum_{i=1}^{n}\phii\pa{r\pa{x_{i},a_{i}} - 	
			\iprod{\phii}{\tvec_{t}}}
			+ \df\Lm\hPsim\vvec_{\tvec_{t},\pivec_{t}}}^{2}_{\Lm^{-1}}\\
		&\qquad\quad\leq
		3\norm{\beta\pa{\rtheta - \tvec_{t}}}^{2}_{\Lm^{-1}}
			+ 3\norm{\frac{1}{n}\sum_{i=1}^{n}\phii\pa{r\pa{x_{i},a_{i}} - 	
				\iprod{\phii}{\tvec_{t}}}}^{2}_{\Lm^{-1}}
			+ 
			3\df^{2}\norm{\Lm\hPsim\vvec_{\tvec_{t},\pivec_{t}}}^{2}_{\Lm^{-1}}.
	\end{align*}
	Now to bound each of the three terms, we use that
	\begin{equation*}
		\norm{\beta\pa{\rtheta - \tvec_{t}}}^{2}_{\Lm^{-1}}
		\leq
		2\beta^{2}\norm{\Lm^{-1}}_{2}\pa{d + D_{\tvec}^{2}}
		\leq  2\beta\pa{d + D_{\tvec}^{2}},
	\end{equation*}
	where the first inequality uses the assumption that $\twonorm{\rtheta}\leq 
	\sqrt{d}$ (cf.~Definition~\ref{def:linMDP}) and $\tvec_{t}\in\bb{d}{D_{\tvec}}$. Next, we have that
	\begin{align*}
		\norm{\frac{1}{n}\sum_{i=1}^{n}\phii\pa{r\pa{x_{i},a_{i}} - 	
				\iprod{\phii}{\tvec_{t}}}}^{2}_{\Lm^{-1}}
			&\leq
			\frac{1}{n}\sum_{i=1}^{n}\norm{\phii}^{2}_{\Lm^{-1}}\abs{r\pa{x_{i},a_{i}}
			 - 	\iprod{\phii}{\tvec_{t}}}^{2}\\
		 	&\leq d\pa{1+RD_{\tvec}}^{2}.
	\end{align*}
	The last step follows from the fact that the rewards are bounded in 
	$[0,1]$, $\norm{\phii}\leq R$, $\tvec_{t}\in\bb{d}{D_{\tvec}}$ and 
	Equation~\eqref{eq:trace_trick}. The last remaining term is bounded as
	\begin{align*}
		\norm{\Lm\hPsim\vvec_{\tvec_{t},\pivec_{t}}}^{2}_{\Lm^{-1}}
		&= 
		\norm{\frac{1}{n}\sum_{i=1}^{n}\phii\vvec_{\tvec_{t},\pivec_{t}}\pa{x_{i}'}}^{2}_{\Lm^{-1}}\\
		&\leq 
		\frac{1}{n}\sum_{i=1}^{n}\norm{\phii}^{2}_{\Lm^{-1}}\infnorm{\vvec_{\tvec_{t},\pivec_{t}}}^{2}
		\leq dR^{2}D_{\tvec}^{2}
	\end{align*}
	Therefore, we obtain
	\begin{align*}
		\norm{\Lm\gl}^{2}_{\Lm^{-1}}
		&\leq
		6\beta\pa{d + D_{\tvec}^{2}} + 3d\pa{1+RD_{\tvec}}^{2} + 3\df^{2} 
		dR^{2}D_{\tvec}^{2}
	\end{align*}
	and this completes the proof.
\end{proof}

\newpage
\section{Missing proofs of 
Section~\ref{appx:GER_proof}}\label{appx:GERlem_proof}
In this section, we prove the lemmas stated in Section~\ref{appx:GER_proof}.

\subsection{Proof of Lemma \ref{lem:LSQerror1}}
By definition of $\Lm$ Section~\ref{sec:FOGAS} and $\hPsim$ in 
Equation~\eqref{eq:P_LSQE}, we can write:
\begin{align*}
	\Lm\pa{\hPsim - \Psim}\vvec
	&= 
	\Lm\pa{\frac{1}{n}\Lm^{-1}\sum_{i=1}^{n}\phii\vec{e}_{x_{i}'}\transpose}\vvec
	-
	\pa{\beta\IIm + 
		\frac{1}{n}\sum_{i=1}^{n}\phii\phii\transpose}\Psim\vvec\\
	&= \frac{1}{n}\sum_{i=1}^{n}\phii[v\pa{x_{i}'} - 
	\iprod{\vec{p}\pa{\cdot|x_{i},a_{i}}}{\vvec}]
	- \beta\Psim\vvec
\end{align*}
In the last equality we used \cref{def:linMDP} to write 
$\phii\transpose\Psim = \vec{p}\pa{\cdot|x_{i},a_{i}}\transpose$. Let 
$\xi_{i} = v\pa{x_{i}'} - \iprod{\vec{p}\pa{\cdot|x_{i},a_{i}}}{\vvec}$. 
Then, 
\begin{align*}
	\norm{\Lm\pa{\hPsim - \Psim}\vvec}_{\Lm^{-1}}
	&\leq
	\norm{\frac{1}{n}\sum_{i=1}^{n}\phii\xi_{i}}_{\Lm^{-1}}
	+
	\norm{\beta\Psim\vvec}_{\Lm^{-1}}.
\end{align*}
We easily control the second term with the relation:
\begin{equation}\label{eq:norm}
	\norm{\beta\Psim\vvec}_{\Lm^{-1}}
	\leq \beta\twonorm{\Lm^{-1/2}}\twonorm{\Psim\vvec}
	\leq B\sqrt{d\beta}
\end{equation}
The last inequality follows from the fact that $\twonorm{\Lm^{-1/2}}\leq 
1/\sqrt{\beta}$ and by \cref{def:linMDP} $\twonorm{\Psim\vvec}\leq 
B\sqrt{d}$ for $\vvec \in [-B,B]^{X}$.

Now, to handle the first term, let $D_{0} = \emptyset$. We construct a 
filtration $\F_{i-1} = 
\D_{i-1} \cup (x^{0}_{i}, x_{i}, a_{i}, r_{i})$ for $i=1, 2, \cdots, n$. 
Notice that by construction of the 	dataset 
$\xi_{i}$ is a martingale difference sequence (i.e 
$\EEc{\xi_{i}}{\F_{i-1}} = 0$) taking values in the range $[-2B,2B]$. 
Then, 
we can directly apply Lemma~\ref{lem:SNBforVVM} to obtain a bound on the 
first 
term as:
\begin{align*}
	\norm{\frac{1}{n}\sum_{i=1}^{n}\phii\xi_{i}}_{\Lm^{-1}}
	=
	\frac{1}{\sqrt{n}}\sqrt{\norm{\sum_{i=1}^{n}\phii\xi_{i}}_{\pa{n\Lm}^{-1}}^{2}}
	&\leq
	\frac{2B}{\sqrt{n}}\sqrt{2\log\Biggl(\frac{\det\pa{n\Lm}^{1/2}\det\pa{n\beta
				I}^{-1/2}}{\delta}\Biggr)}\\
	&\leq 
	\frac{2B}{\sqrt{n}}\sqrt{d\log\Biggl(1 + \frac{R^{2}}{d\beta}\Biggr) + 
		2\log\frac{1}{\delta}}.
\end{align*}
with probability $1-\delta$. In the last inequality we have used the AM-GM 
inequality and bound on the feature vectors:
\[
\det\pa{n\Lm}
\leq \pa{\frac{\trace{n\Lm}}{d}}^{d}\\
= \pa{n\beta + 
	\frac{\trace{\sum_{i=1}^{n}\phii\phii\transpose}}{d}}^{d}\\
\leq \pa{n\beta + 
	\frac{nR^{2}}{d}}^{d}.
\]
Putting everything together, we have that w.p $1-\delta$,
\begin{align*}
	\norm{\Lm\pa{\hPsim - \Psim}\vvec}_{\Lm^{-1}}
	&\leq
	\frac{2B}{\sqrt{n}}\sqrt{d\log\Biggl(1 + \frac{R^{2}}{d\beta}\Biggr) + 
		2\log\frac{1}{\delta}}
	+
	B\sqrt{d\beta}.
\end{align*}
This completes the proof.\hfill\qed

\subsection{Proof of Lemma \ref{lem:LSQerror2}}
Unlike Lemma~\ref{lem:LSQerror1}, we now aim to control the error term 
$\norm{\Lm\pa{\hPsim - \Psim}\vvec}_{\Lm^{-1}}$ when $\vvec$ is random. 
Also, notice that with $\pi_{1}\pa{a|x} 
= 
\frac{e^{\iprod{\feat}{\vec{0}}}}{\sum_{a'\in\A}e^{\iprod{\varphi\pa{x,a'}}{\vec{0}}}}$
as the uniform policy, for $t=1,\cdots,T$ we have that,
\[
	\pi_{t+1}(a|x)
	= \frac{\pi_{1}(a|x) e^{\alpha 
		\iprod{\varphi(x,a)}{\sum_{k=1}^t \tvec_k}}}{\sum_{a'}\pi_{1}(a'|x) 
		e^{\alpha 
		\iprod{\varphi(x,a')}{\sum_{k=1}^t \tvec_k}}}
	= \frac{e^{\iprod{\varphi(x,a)}{\alpha\sum_{k=0}^t \tvec_k}}}{\sum_{a'} 
		e^{\iprod{\varphi(x,a')}{\alpha\sum_{k=0}^t \tvec_k}}},
\]
where $\tvec_{0}=\vec{0}$. Furthermore, since 
$\{\tvec_{t}\}_{t=1}^{T}\subset\bb{d}{D_{\tvec}}$, for any $t$ 
$\twonorm{\alpha\sum_{k=0}^t \tvec_k}\leq \alpha TD_{\tvec}$. 
Hence, with $D_{\pi}=\alpha TD_{\tvec}$, $\pi_{t}\in\Pi\pa{D_{\pi}}$ and 
$\vvec_{\theta_{t},\pi_{t}}\in\VV$.

Therefore, as we have seen in previous works
\citep{jin2020provably,hong2024primal}, 
the quantity $\norm{\Lm\pa{\hPsim - 
\Psim}\vvec_{\theta_{t},\pi_{t}}}_{\Lm^{-1}}$ can be controlled without any 
dependence on the size of the state space with a \emph{uniform covering} 
argument over $\VV$. Let $C_{\vvec}$ be an 
$\epsilon$-cover of $\VV$. That is, for 
$\vvec_{\pivec_{t},\tvec_{t}}\in\VV$, there exists 
$\vvec'\in C_{\vvec}$ such that $\infnorm{\vvec_{\pivec,\tvec_{t}} - 
	\vvec'}\leq \epsilon$. Then, we can write:
\begin{align}
	\nonumber&\norm{\Lm\pa{\hPsim - 
	\Psim}\vvec_{\theta_{t},\pi_{t}}}_{\Lm^{-1}}\\
	\label{eq:selfnormproc}&\qquad\quad\leq
	\norm{\Lm\pa{\hPsim - \Psim}\vvec'}_{\Lm^{-1}}
	+
	\norm{\Lm\hPsim\pa{\vvec_{\theta_{t},\pi_{t}} - 
			\vvec'}}_{\Lm^{-1}}
	+
	\norm{\Lm\Psim\pa{\vvec' - \vvec_{\theta_{t},\pi_{t}}}}_{\Lm^{-1}}
\end{align}
Consider the first term in the bound. Note that $\vvec'$ is still 
random with respect to uncertainty in the learning process. However, due to 
the structure of $\VV$ we know that $C_{\vvec}$ exists and has cardinality 
$\log|C_{\vvec}| = 
\mathcal{O}\pa{d\log\pa{1+\frac{4RD_{\pi}RD_{\tvec}}{\epsilon}}}$ (see 
Lemma~\ref{lem:value_cover}). Inspired by Lemma~\ref{lem:LSQerror1}, consider 
the event:
\[
\mathcal{E}_{\vvec} = \Biggl\{\text{exists}\,\vvec\in C_{\vvec}:
\norm{\Lm\pa{\hPsim - \Psim}\vvec}_{\Lm^{-1}}
>
\frac{2RD_{\tvec}}{\sqrt{n}}\sqrt{d\log\Biggl(1 
	+ 
	\frac{R^{2}}{d\beta}\Biggr) + 
	2\log\frac{1}{\delta'}}
+ RD_{\tvec}\sqrt{d\beta}\Biggr\}
\]
Since $C_{\vvec}\subseteq\VV$, we know from Lemma~\ref{lem:LSQerror1} that 
$\mathbb{P}\pa{\mathcal{E}_{\vvec}} \leq \delta'$. Now, taking the union bound 
over the cover $C_{\vvec}$ we have that,
\[
\mathbb{P}\pa{\bigcup_{\vvec\in C_{\vvec}}\mathcal{E}_{\vvec}}
\leq |C_{\vvec}|\delta'.
\]
Therefore for any $\vvec'\in C_{\vvec}$ with probability at least 
$1-\delta$,
\begin{align*}
	&\norm{\Lm\pa{\hPsim - \Psim}\vvec'}_{\Lm^{-1}}\\
	&\qquad\leq
	\frac{2RD_{\tvec}}{\sqrt{n}}\sqrt{d\log\Biggl(1 + 
		\frac{R^{2}}{d\beta}\Biggr) + 
		2\log\frac{|C_{\vvec}|}{\delta}}
	+
	RD_{\tvec}\sqrt{d\beta}\\
	&\qquad\leq
	\frac{2RD_{\tvec}}{\sqrt{n}}\sqrt{d\log\Biggl(1 + 
		\frac{R^{2}}{d\beta}\Biggr) + 
		4d\log\pa{1+\frac{4RD_{\pi}RD_{\tvec}}{\epsilon}}+ 
		2\log\frac{1}{\delta}}
	+
	RD_{\tvec}\sqrt{d\beta}
\end{align*}
Now, for the second term in Equation~\eqref{eq:selfnormproc} we write,
\begin{align*}
	\norm{\Lm\hPsim\pa{\vvec_{\theta_{t},\pi_{t}} - 
			\vvec'}}_{\Lm^{-1}}^{2}
	&= 
	\norm{\frac{1}{n}\sum_{i=1}^{n}\phii\pa{\vvec_{\theta_{t},\pi_{t}}\pa{x'_{i}}
			- 
			\vvec'\pa{x'_{i}}}}_{\Lm^{-1}}^{2}\\
	&\overset{(a)}{\leq} 
	\frac{1}{n}\sum_{i=1}^{n}\abs{\vvec_{\theta_{t},\pi_{t}}\pa{x'_{i}}
		- 
		\vvec'\pa{x'_{i}}}^{2}\norm{\phii}_{\Lm^{-1}}^{2}\\
	&\leq 
	\epsilon^{2}\frac{1}{n}\sum_{i=1}^{n}\norm{\phii}_{\Lm^{-1}}^{2}
	\overset{(b)}{\leq}  \epsilon^{2}d.
\end{align*}
We have used $(a)$ Jensen's inequality and $(b)$ since $\Lm\succ 
0$, the relation,
\begin{equation}\label{eq:trace_trick}
	\frac{1}{n}\sum_{i=1}^{n}\phii\transpose\Lm^{-1}\phii
	= \frac{1}{n}\sum_{i=1}^{n}\trace{\Lm^{-1}\phii\phii\transpose}
	= \trace{\Lm^{-1}\frac{1}{n}\sum_{i=1}^{n}\phii\phii\transpose}
	\leq \trace{I} = d.
\end{equation}
For the last term, notice that:
\begin{align*}
	\norm{\Lm\Psim\pa{\vvec' - \vvec_{\theta_{t},\pi_{t}}}}_{\Lm^{-1}}
	&=
	\norm{\beta\Psim\pa{\vvec' - 
			\vvec_{\theta_{t},\pi_{t}}} + 
		\frac{1}{n}\sum_{i=1}^{n}\phii\Bigl[\sum_{x'}p\pa{x'|x_{i},a_{i}}\pa{\vvec'\pa{x'}
			- 
			\vvec_{\theta_{t},\pi_{t}}\pa{x'}}\Bigr]}_{\Lm^{-1}}\\
	&\overset{(a)}{\leq} \epsilon\sqrt{d\beta} + 
	\sqrt{\norm{\frac{1}{n}\sum_{i=1}^{n}\phii\Bigl[\sum_{x'}p\pa{x'|x_{i},a_{i}}\pa{\vvec'\pa{x'}
				- 
				\vvec_{\theta_{t},\pi_{t}}\pa{x'}}\Bigr]}_{\Lm^{-1}}^{2}}\\
	&\overset{(b)}{\leq} \epsilon\sqrt{d\beta} + 
	\sqrt{\frac{1}{n}\sum_{i=1}^{n}\infnorm{\vvec'- 
			\vvec_{\theta_{t},\pi_{t}}}^{2}\norm{\phii}_{\Lm^{-1}}^{2}}\\
	&\overset{(c)}{\leq}  \epsilon\sqrt{d\beta} + \epsilon\sqrt{d} = 
	\epsilon\sqrt{d}\pa{\sqrt{\beta} + 1}.
\end{align*}
This follows from $(a)$ Equation~\eqref{eq:norm} since $\vvec = \vvec' - 
\vvec_{\theta_{t},\pi_{t}} \in [-\epsilon,\epsilon]^{X}$ and $(b)$ 
monotonicity of the square root function as well as Jensen's inequality and 
$(c)$ Equation~\eqref{eq:trace_trick}.

Finally, plugging the above results back into Equation~\eqref{eq:selfnormproc}, 
we have that with probability at least $1-\delta$,
\begin{align*}
	&\norm{\Lm\pa{\hPsim - \Psim}\vvec_{\theta_{t},\pi_{t}}}_{\Lm^{-1}}\\
	&\qquad\quad\leq
	\frac{2RD_{\tvec}}{\sqrt{n}}\sqrt{d\log\Biggl(1 + 
		\frac{R^{2}}{d\beta}\Biggr) + 
		4d\log\pa{1+\frac{4RD_{\pi}RD_{\tvec}}{\epsilon}} + 
		2\log\frac{1}{\delta}}\\
	&\quad\qquad\quad
	+
	RD_{\tvec}\sqrt{d\beta}
	+
	\pa{\sqrt{\beta} + 1}\epsilon\sqrt{d}
\end{align*}
The proof of Lemma~\ref{lem:LSQerror2} is complete.\qed

\newpage
\section{Auxiliary Lemmas}\label{app:aux}
\begin{lemma}\label{lem:MD}
	Let $\qvec_{1},\cdots,\qvec_{t}$ be a sequence of iterates satisfying 
	$\infnorm{\qvec_{t}}\leq RD_{\tvec}$ by virtue of \cref{def:linMDP} and 
	$\tvec_{t}\in\bb{d}{D_{\tvec}}$. Given an initial policy $\pi_{1}$ and 
	learning rate $\alpha>0$, and sequence of policies $\{\pi_{t}\}_{t=2}^{T}$ 
	defined as:
	\[
		\pi_{t+1}(a|x) = \frac{\pi_{t}(a|x)e^{\alpha 
		q_{t}(x,a)}}{\sum_{a'}\pi_{t}(a'|x)e^{\alpha q_{t}(x,a')}},
	\]
	Then, for any comparator policy $\pi^{*}$ and $\nu^{*}$ some
	state distribution,
	\[
		\sum_{t=1}^{T}
		\sum_{x}\nu^{*}(x)\sum_{a}\pa{\pi^{*}(a|x) - \pi_{t}(a|x)}q_{t}(x,a)
		\leq
		\frac{\sum_{x}\nu^{*}(x)\DDKL{\pi^{*}\pa{\cdot|x}}{\pi_{1}\pa{\cdot|x}}}{\alpha}
		 + \frac{\alpha 
		TR^{2}D_{\tvec}^{2}}{2}.
	\]
\end{lemma}
The proof of the lemma follows from bounding the regret of the $\pi$-player in each state $x$
as
\[
 \sum_{t=1}^T\sum_{a}\pa{\pi^{*}(a|x) - \pi_{t}(a|x)}q_{t}(x,a) \le 
\frac{\DDKL{\pi^{*}\pa{\cdot|x}}{\pi_{1}\pa{\cdot|x}}}{\alpha}
		 + \frac{\alpha}{2} \sum_{t=1}^T \infnorm{q_t(x,\cdot)}^2,
\]
via the application of the standard analysis of the exponentially weighted forecaster of \citet{Vov90,LW94,FS97} (see, 
e.g., Theorem 2.2 in \citealp{cesa2006prediction}), and noting that $\infnorm{q_t} \le RD_\theta$ for all $t$.

\begin{lemma}\label{lem:policy_cover}
	Suppose that $\twonorm{\feat}\leq R$ for all $(x,a)\in\X\times\A$. Let 
	$\pi_{\tvec},\pi_{\tvec'}$ be softmax policies. Then, for all states 
	$x\in\X$ we have that:
	\[
	\sum_{a}\abs{\pi_{\tvec}\pa{a|x} - \pi_{\tvec'}\pa{a|x}}
	\leq R\twonorm{\tvec - \tvec'}
	\]
	holds for any $\tvec,\tvec'\in\Rn^{d}$.
\end{lemma}
\begin{proof}
	Recall that,
	\[
	\Pi\pa{D_{\pi}} = 
	\Biggl\{\pi_{\theta}\pa{a|x} = 
	\frac{e^{\iprod{\feat}{\tvec}}}{\sum_{a'}e^{\iprod{\bm{\varphi}\pa{x,a'}}{\tvec}}}
	\,\Bigg|\,\tvec\in\bb{d}{D_{\pi}}\Biggr\}.
	\]
	For $\pivec_{\tvec},\pivec_{\tvec'}\in\Pi\pa{D_{\pivec}}$ using 
	Pinsker's inequality we have that,
	\begin{equation}\label{eq:pinskerineq}
		\onenorm{\pivec_{\tvec}\pa{\cdot|x} - \pivec_{\tvec'}\pa{\cdot|x}}
		\leq 
		\sqrt{2\DDKL{\pivec_{\tvec}\pa{\cdot|x}}{\pivec_{\tvec'}\pa{\cdot|x}}}
		\qquad
		\text{ for } x\in\X.
	\end{equation}
	Furthermore, taking into account the specific structure of the policies, we 
	can write:
	\begin{align*}
		\DDKL{\pivec_{\tvec}\pa{\cdot|x}}{\pivec_{\tvec'}\pa{\cdot|x}}
		&=
		\sum_{a}\pivec_{\tvec}\pa{a|x}\log\frac{\pivec_{\tvec}\pa{a|x}}{\pivec_{\tvec'}\pa{a|x}}\\
		&=
		- \sum_{a}\pivec_{\tvec}\pa{a|x}\iprod{\feat}{\tvec'-\tvec}
		+ 
		\log\frac{\sum_{a}e^{\iprod{\feat}{\tvec'}}}{\sum_{a}e^{\iprod{\feat}{\tvec}}}\\
		&\overset{(a)}{=}
		- \sum_{a}\pivec_{\tvec}\pa{a|x}\iprod{\feat}{\tvec'-\tvec}
		+ 
		\log\sum_{a}\pivec_{\tvec}\pa{a|x}e^{\iprod{\feat}{\tvec'-\tvec}}\\
		&\overset{(b)}{=}
		\frac{R^{2}\sqtwonorm{\tvec - \tvec'}}{2}
	\end{align*}
	using that $(a)$ the relation,
	\begin{equation*}
		\log\frac{\sum_{a}e^{\iprod{\feat}{\tvec'}}}{\sum_{a}e^{\iprod{\feat}{\tvec}}}
		= 
		\log\sum_{a}\frac{e^{\iprod{\feat}{\tvec'}}}{\sum_{a'}e^{\iprod{\bm{\varphi}\pa{x,a'}}{\tvec}}}
		\cdot \frac{e^{\iprod{\feat}{\tvec}}}{e^{\iprod{\feat}{\tvec}}}
		=		
		\log\sum_{a}\pivec_{\tvec}\pa{a|x}e^{\iprod{\feat}{\tvec'-\tvec}},
	\end{equation*}
	and $(b)$ Hoeffding's lemma (cf. Lemma A.1 of \cite{cesa2006prediction}). 
	The final statement follows from substituting this result in 
	Equation~\eqref{eq:pinskerineq}.
\end{proof}

\begin{lemma}(Self-Normalized Bound for Vector-Valued Martingales - Theorem 1 
	of \citet{abbasi2011improved})\label{lem:SNBforVVM}
	Let $\{\F_{i-1}\}_{i=1}^{\infty}$ be a filtration and  
	$\{\xi_{i}\}_{i=1}^{\infty}$ a \emph{real-valued} stochastic process such 
	that $\xi_{i}$ for $i=1,\cdots$ is zero-mean (i.e $\EEc{\xi_{i}}{\F_{i-1}} 
	= 0$) and conditionally $s$-subgaussian for $s\geq 0$. That is, for all 
	$b\in\Rn$,
	\[
	\EEc{e^{b\xi_{i}}}{\F_{i-1}}\leq e^{\frac{b^{2}s^{2}}{2}}.
	\]
	Also, let $\{\phii\}_{i=1}^{\infty}$ be 
	$\F_{i-1}$-measurable. Then,
	\[
	\norm{\sum_{i=1}^{n}\phii\xi_{i}}_{\pa{n\Lm}^{-1}}^{2}
	\leq
	2s^{2}\log\Biggl[\frac{\det\pa{n\Lm}^{1/2}\det\pa{n\beta 
	I}^{-1/2}}{\delta}\Biggr].
	\]
\end{lemma}

\begin{lemma}(e.g. see Chapter 27 of 
\citet{shalev2014understanding})
	For all $\epsilon>0$,
	\[
	\log\N\pa{\bb{d}{r}, \infnorm{\cdot},\epsilon}
	\leq
	d\log\pa{1+\frac{2r}{\epsilon}}.
	\]
\end{lemma}

\begin{corollary}
	Under the conditions of Lemma~\ref{lem:policy_cover}, for all $\epsilon>0$,
	\[
	\log\N\pa{\Pi\pa{D_{\pi}}, \norm{\cdot}_{\infty,1},\epsilon}
	\leq
	\log\N\pa{\bb{d}{D_{\pi}}, \infnorm{\cdot},\frac{\epsilon}{R}}
	\leq
	d\log\pa{1+\frac{2RD_{\pi}}{\epsilon}}.
	\]
\end{corollary}
\begin{lemma}\label{lem:value_cover}
	Consider the function class,
	\[
	\VV = \Bigl\{\vvec_{\pivec,\tvec}:\X\rightarrow 
	[-RD_{\tvec},RD_{\tvec}] \Big| 
	\pivec\in\Pi\pa{D_{\pivec}}, \tvec\in\bb{d}{D_{\tvec}}\Bigr\},
	\]
	we have that:
	\[
	\N\pa{\VV, \infnorm{\cdot},\epsilon}
	\leq
	\N\pa{\Pi\pa{D_{\pi}}, \norm{\cdot}_{\infty,1},\epsilon/2RD_{\tvec}}
	\times
	\N\pa{\bb{d}{D_{\tvec}}, \twonorm{\cdot},\epsilon/2R},
	\]
	and,
	\[
	\log\N\pa{\VV, \infnorm{\cdot},\epsilon}
	\leq
	2d\log\pa{1+\frac{4RD_{\pi}RD_{\tvec}}{\epsilon}}
	\]	
\end{lemma}
\begin{proof}
	Let $C_{\pivec}$ denote the	
	$\epsilon_{\pivec}$-cover of 
	$\Pi\pa{D_{\pi}}$ with respect to the norm $\norm{\cdot}_{\infty, 1}$ and 
	$C_{\tvec}$ the	$\epsilon_{\tvec}$-cover of $\bb{d}{D_{\tvec}}$ under 
	the $L^2$-norm. For 
	$\pa{\pivec,\tvec}\in\Pi\pa{D_{\pi}}\times\bb{d}{D_{\tvec}}$ and  
	$\pa{\pivec',\tvec'}\in C_{\pivec}\times C_{\tvec}$, 
	it follows that for any state $x\in\X$,
	\begin{align*}
		\abs{\vvec_{\pivec,\tvec}(s) - \vvec_{\pivec',\tvec'}(s)}
		&=
		\abs{\sum_{a\in\A}
			\pivec(a|x)\iprod{\feat}{\tvec}
			- \pivec'(a|x)\iprod{\feat}{\tvec'}}\\
		&=
		\abs{\sum_{a\in\A}
			\pa{\pivec(a|x)- \pivec'(a|x)}\iprod{\feat}{\tvec}
			+ \sum_{a\in\A}\pivec'(a|x)\iprod{\feat}{\tvec - \tvec'}}\\
		&\leq
		RD_{\tvec}\sum_{a\in\A}
		\abs{\pivec(a|x)- \pivec'(a|x)}
		+ R\sum_{a\in\A}\pivec'(a|x)\twonorm{\tvec - \tvec'}
	\end{align*}
	Let $C_{\vvec} = \Bigl\{\vvec_{\pivec,\tvec}:\X\rightarrow 
	[-RD_{\tvec},RD_{\tvec}] \Big| \pivec\in C_{\pivec}, \tvec\in 
	C_{\tvec}\Bigr\}$. 
	Then, $C_{\vvec}$ is an $\epsilon$-cover of $\VV$ with respect to the 
	$L^{\infty}$-norm when $\epsilon_{\pivec} = \epsilon/2RD_{\tvec}$ and 
	$\epsilon_{\tvec}=\epsilon/2R$. Therefore, we can derive a bound on 
	the covering number of $C_{\vvec}$ as:
	\begin{align*}
		\N\pa{\VV, \infnorm{\cdot},\epsilon}
		&\leq
		\N\pa{\Pi\pa{D_{\pi}}, \norm{\cdot}_{\infty,1},\epsilon/2RD_{\tvec}}
		\times
		\N\pa{\bb{d}{D_{\tvec}}, \twonorm{\cdot},\epsilon/2R}\\
		&\leq
		\pa{1+\frac{4RD_{\pi}RD_{\tvec}}{\epsilon}}^{d}
		\pa{1+\frac{4RD_{\tvec}}{\epsilon}}^{d}.
	\end{align*}
	Hence,
	\begin{align*}
		\log\N\pa{\VV, \infnorm{\cdot},\epsilon}
		&\leq 
		2d\log\pa{1+\frac{4RD_{\pi}RD_{\tvec}}{\epsilon}}
	\end{align*}
	This completes the proof.
\end{proof}

\end{document}